\documentclass[twoside]{article}

\usepackage{amssymb}
 
\usepackage{amsmath}
\usepackage{algorithm2e}
\usepackage{empheq}
\usepackage{amsthm}
\usepackage[toc,page]{appendix} 
\usepackage{pgfplots}
\usepackage{multicol}
\usepackage{fancybox}

\usepackage{graphicx,subcaption}
\usepackage{duckuments}

 \newtheorem{theorem}{Theorem}
\newtheorem{definition}{Definition}

\newtheorem{lemma}{Lemma}

\newlength\dlf  

\usepackage{algpseudocode}

\usepackage[accepted]{aistats2022}
\begin{document}

%

%

\twocolumn[

\aistatstitle{Private measurement of nonlinear correlations between data hosted across multiple parties}

\aistatsauthor{ Praneeth Vepakomma \And Subha Nawer Pushpita \And  Ramesh Raskar }

\aistatsaddress{ MIT \And  MIT\And MIT } ]

\begin{abstract}
We introduce a differentially private method to measure nonlinear correlations between sensitive data hosted across two entities. We provide utility guarantees of our private estimator. Ours is the first such private estimator of nonlinear correlations, to the best of our knowledge within a multi-party setup. The important measure of nonlinear correlation we consider is \textit{distance correlation}. This work has direct applications to  private feature screening, private independence testing, private k-sample tests, private multi-party causal inference and private data synthesis in addition to exploratory data analysis. \textbf{Code access:} A link to publicly access the code is provided in the supplementary file.
 \end{abstract}

\section{Introduction}
Estimating correlations (linear and non-linear) between random variables (via their samples) is of fundamental importance in statistics and machine learning. The following is the main problem of consideration in this paper.\\ \textbf{Problem statement:}  How can non-linear correlations between two random variables be estimated with formal guarantees on preserving privacy while reasonably maintaining the utility of the estimate under the constraint that the samples from each of the random variables are correspondingly held by two different parties?  

 \begin{figure*}
    \centering
    \includegraphics{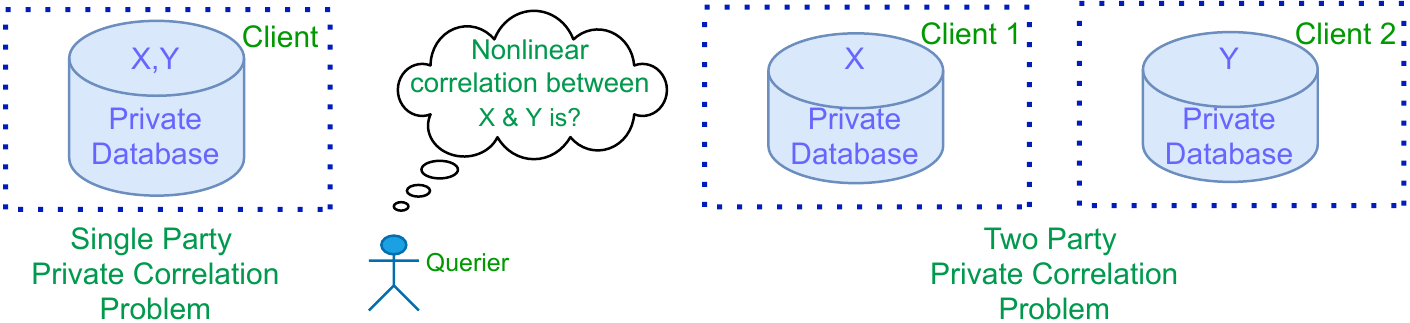}
    \caption{Private multi-party non-linear correlation problem}
    \label{fig:pncProblem}
\end{figure*}
\subsection{Motivation}
An ability to compute non-linear correlations between data hosted across multiple parties while preserving privacy opens up quite a few applications. The dependency measure we consider in this paper is called distance correlation \cite{szekely2007measuring} which is an instance of `Energy Statistics' introduced in \cite{szekely2017energy,rizzo2016energy,szekely2013energy}. Some of the important applications for estimating distance correlation in a private manner are summarized below.

\subsubsection{Benefits of privately estimating distance correlation}  

\begin{enumerate}
    
    \item \textbf{Private multi-party feature screening with distance correlation} The importance of distance correlation in optimally selecting features with a `\textit{sure independence screening}' guarantee under a model agnostic setting was shown in the works of \cite{li2012feature,zhong2015iterative,fan2018sure,lu2021conditional}. Thereby the measurement of distance correlation in a privacy preserving manner in multi-party settings would potentially allow for private feature selection. A noteworthy work by \cite{thakurta2013differentially} shows that under assumptions of feature selection procedures being \textit{stable} with respect to sub-sampling, the selection can be done privately and in some cases match the optimal non-private asymptotic sample complexity. 
   
   
    \item \textbf{Private multi-party independence testing with distance correlation}
    A wonderful survey of property testing of distributions is given in \cite{canonne2020survey} and also covers the specific sub-topic of independence testing. 
    Hypothesis tests for testing independence of distributions using distance correlation were introduced in \cite{szekely2007measuring,szekely2013distance}. The test statistic is based on distance correlation. Thereby performing private independence testing between samples that are distributed in multiple entities requires private estimation of distance correlation. The classical Chi-Squared hypothesis test for testing independence between categorical variables has been extended to be made differentially private in \cite{gaboardi2016differentially}. The work in \cite{shen2019exact} showed an equivalence between independence testing using distance correlation and k-sample testing. Differentially private procedures for analysis of variance (ANOVA) have been introduced in \cite{swanberg}.
   


    \item \textbf{Private multi-party causal inference with distance correlation} The works in
\cite{kusner2015inferring,kusner2016private} privately estimate distance correlation in a single party setting, where samples from both the random variables are at at the same entity. They then use this private estimate of distance correlation to infer the causal direction. We consider the more important setting where the samples from both the random variables are at two different corresponding entities as opposed to all of them being on-premise at one entity.
\item \textbf{Private multi-party data synthesis by seeding copulas with distance correlations} Gaussian copulas have been used for private data synthesis as shown in
\cite{Asghar_Ding_Rakotoarivelo_Mrabet_Kaafar_2021,li2014differentially}. A key step in this process is to seed the Gaussian copula with a correlation matrix, and it has a significant impact on the quality of the synthesis. These current works have so far seeded it with linear measures of correlations such as Pearson's correlation or Kendall's Tau. A private multi-party measure of distance correlation allows for better seeding of the Gaussian copula for multi-party private data synthesis.
\end{enumerate}

\subsection{Contributions}
\begin{enumerate}
    \item We introduce a differentially private method to measure nonlinear correlations between sensitive data hosted across two entities.  
    \item We provide utility guarantees for our private estimator. We provide experimental results to compare the quality of our estimator against important baselines on several benchmark datasets.
\end{enumerate}

   







\section{Preliminaries}
We start by providing definitions of population distance covariance and population distance correlation as these are central to rest of the paper. We then share two existing non-private sample estimators for estimating the same. 
\subsection{Non-private estimators of distance correlation}

\textbf{Population distance covariance} 

	For random variables $\mathbf{x} \in \mathbb{R}^d$ and $\mathbf{y} \in \mathbb{R}^m$ with finite first moments, the population distance covariance \cite{szekely2007measuring} between them is a non-negative number given by
$$
		\mathbb{\nu}^2(\mathbf{x},\mathbf{y})=\int_{\mathbb{R}^{d+m}}|f_{\mathbf{x},\mathbf{y}}(t,s)-f_\mathbf{x}(t)f_\mathbf{y}(s)|^2 w(t,s)dtds
$$
where $f_\mathbf{x},f_\mathbf{y}$ are characteristic functions of $\mathbf{x},\mathbf{y}$, $f_{\mathbf{x},\mathbf{y}}$ is the joint characteristic function, and $w(t, s)$ is a weight function defined as $$w(t, s) = (C(p,\alpha) C(q,\alpha) |t|^{\alpha+p}_p |s|_q^{\alpha+q})^{-1}$$ with $$C(d,\alpha) = \frac{2\pi^{d/2}\Gamma(1-\alpha/2)}{\alpha2^{\alpha}\Gamma((\alpha + d)/2)}$$ for chosen values of $\alpha$ which refers to the choice of norm considered in obtaining the distance matrices. $\Gamma$ refers to the popular complete Gamma function, that is defined to be an extension of the concept of a factorial to complex and real numbers as opposed to just the integers. Note that for random variables that admit a density, the characteristic function is the Fourier transform of the probability density function. \\\textbf{Note:} We would like to note that distance covariance between a variable and itself is referred to as distance variance. 
\\\textbf{Population distance correlation} Using this above definition of distance covariance, we have the following expression for the square of distance correlation \cite{szekely2007measuring} between random variables \[
	\left\{ \begin{array}{cc}
				  \frac{{\Omega}(\mathbf{X},\mathbf{Y})}{\sqrt{{\Omega}(\mathbf{X},\mathbf{X}){\Omega}(\mathbf{Y},\mathbf{Y})}}, 
				 & \text{if, }{\Omega}(\mathbf{X},\mathbf{X}){\Omega}(\mathbf{Y},\mathbf{Y})>0\\
	        0, & \text{if, } {\Omega}(\mathbf{X},\mathbf{X}){\Omega}(\mathbf{Y},\mathbf{Y})=0
	        \end{array} \right.
	\]

This always lies within the interval $[0,1]$ with $0$ referring to independence and $1$ referring to dependence.
\par For completeness, we provide a listing of some more popular measures of nonlinear statistical dependency in Appendix \ref{appdx:Families}. That said, this paper specifically focuses on mechanisms for privatizing the measure of distance correlation in settings when the variables $\mathbf{X}$, $\mathbf{Y}$ are hosted at different entities. In the setting when both of them are hosted on the same entity, a privatized measure of a statistical dependency called Hilbrt Schmidt Indpendence Critrion was provided in \cite{kusner2016private}. 

\subsection{Sample estimators of distance correlation}
Two sample estimators currently available for estimating population distance covariance in the non-private setting include:

   \textbf{a.) Unbiased sample distance covariance} 
   \\Using $|\cdot|$ to represent the Euclidean norm, we first define  $a_{ij} = |X_i-X_j|$, $b_{ij} = |Y_i-Y_j|$,  $a_{i}. =\sum_{l=1}^{n} a_{i l}$,   $b_{i} .=\sum_{l=1}^{n} b_{i l}$, $a . . =\sum_{k, l=1}^{n} a_{k l}$ and $b . .=\sum_{k, l=1}^{n} b_{k l} $. We now use these quantities to define an unbiased statistical estimator of distance covariance $\hat{\Omega}(\mathbf{X,Y})$ as follows. 
   \begin{align}\label{eq:dcovAB}\hat{\Omega}(\mathbf{X,Y})&=\frac{1}{n(n-3)} \sum_{i \neq j} a_{ij} b_{ij} - \frac{2}{n(n-2)(n-3)}\nonumber\\ &\sum_{i=1}^n a_{i\cdot} b_{i\cdot}  + \frac{a_{..}b_{..}}{n(n-1)(n-2)(n-3)}   \end{align}

   
 \textbf{b.) Random projected distance covariance}
    A faster unbiased estimator of distance covariance  based on random projections denoted by $\overline{\Omega^{k}_{n}}(\mathbf{X,Y})$ was introduced in \cite{huang2017statistically}. 
     In order to define this sample estimator of distance covariance, we first define few constants based on $\pi$ and the Gamma function as follows. These include
     $c_{p}=\frac{\pi^{(p+1) / 2}}{\Gamma((p+1) / 2)}$ and $c_{q}=\frac{\pi^{(q+1) / 2}}{\Gamma((q+1) / 2)}$ , $C_{p}=\frac{c_{1} c_{p-1}}{c_{p}}=\frac{\sqrt{\pi} \Gamma((p+1) / 2)}{\Gamma(p / 2)}$ and $C_{q}=\frac{c_{1} c_{q-1}}{c_{q}}=\frac{\sqrt{\pi} \Gamma((q+1) / 2)}{\Gamma(q / 2)} .$ Let $u$ and $v$ be points on the hyper-spheres: $u \in \mathcal{S}^{p-1}=$ $\left\{u \in \mathbb{R}^{p}:|u|=1\right\}$ and $v \in \mathcal{S}^{q-1}$. For any vector $u$ or $v$, let $u^{T}$ or $v^{T}$  denote its transpose. We now share the non-private estimator below
     \begin{equation}
    \overline{\Omega}(\mathbf{X,Y})= \sum_{k=1}^{K} \frac{C_pC_q\Omega_n(u_k^TX,v_k^TY)}{K}     \label{eq:rpdc}
     \end{equation}
     where $u_k^TX = (u_k^TX_1 , \ldots, u_k^TX_n)$ and $u_k^TY = (u_k^TY_1 , \ldots,u_k^TY_n)$. Note that in this case $a_{ij}= |u^T(X_i-X_j)|$ and $b_{ij}= |v^T(Y_i-Y_j)|$.
\section{Method}
\subsection{Approach}  Privately estimating sample distance correlation between $\mathbf{X}\in \mathbb{R}^{n \times d}$ and $\mathbf{Y}\in \mathbb{R}^{n \times p}$ requires a private estimation of a distance covariance term in the numerator and a distance variance term in the denominator of $ \frac{\overline{{\Omega}}(\mathbf{X},\mathbf{Y})}{\sqrt{{\overline{\Omega}}(\mathbf{X},\mathbf{X}){\overline{\Omega}}(\mathbf{Y},\mathbf{Y})}}$ that depends on $\mathbf{X}$.

\textbf{Privatizing ${\overline{\Omega}\mathbf{(X,Y)}}$: }For the distance covariance in the numerator, we provide a differentially private estimator of ${\overline{\Omega}\mathbf{(X,Y)}}$ that we denote by ${\overline{\Omega}^{dp}\mathbf{(X,Y)}}$. Note that ${\overline{\Omega}\mathbf{(X,Y)}}$ is the non-private `random projected distance covariance estimator' given in equation \ref{eq:rpdc} to estimate population distance covariance ${{\Omega}\mathbf{(X,Y)}}$. Note that in the setting we consider, $\mathbf{X}$ and  $\mathbf{Y}$ are hosted at two different entities, as opposed to both of them being available on-premise at a single site. We provide details on our differentially private estimator ${\overline{\Omega}^{dp}\mathbf{(X,Y)}}$ in section 3.2.1 and provide our utility proofs for it in Theorems 1, 2 \& 3. 

\textbf{Privatizing ${\overline{\Omega}\mathbf{(Y,Y)}}$: } The term of $\overline{\Omega}(\mathbf{Y},\mathbf{Y})$ does not require private estimation as it is computed on-premise by the entity that holds $\mathbf{Y}$. The entity that holds $\mathbf{X}$ (say, Alice) makes a one-way communication of our proposed private estimates for  ${\overline{\Omega}^{dp}\mathbf{(X,Y)}}$,  ${\overline{\Omega}^{dp}\mathbf{(X,X)}}$ to the entity that holds $\mathbf{Y}$ (say, Bob). Bob computes a non-private and  ${\overline{\Omega}\mathbf{(Y,Y)}}$ on-premise. These three estimates can be put together by Bob to privately estimate the distance correlation $\overline{\rho}^{dp}\mathbf{(X,Y)}$ as 	\[
	\left\{ \begin{array}{cc}
				  \frac{\overline{\Omega}^{dp}(\mathbf{X},\mathbf{Y})}{\sqrt{\overline{\Omega}^{dp}(\mathbf{X},\mathbf{X})\overline{\Omega}^{dp}(\mathbf{Y},\mathbf{Y})}}, 
				 & \text{if, }\overline{\Omega}^{dp}(\mathbf{X},\mathbf{X})\overline{\Omega}^{dp}(\mathbf{Y},\mathbf{Y})>0\\
	        0, & \text{if, } \overline{\Omega}^{dp}(\mathbf{X},\mathbf{X})\overline{\Omega}^{dp}(\mathbf{Y},\mathbf{Y})=0
	        \end{array} \right.
	\] 
	Note that Bob does not reveal the estimated private distance correlation to Alice. In a scenario, where it needs to reveal it to Alice (or anyone in the public), it could do that by using the private estimator in the subsection on privatizing ${\overline{\Omega}\mathbf{(X,X)}}$ with respect to $\mathbf{X}$. This approach can be used by Bob to privatize ${\overline{\Omega}\mathbf{(Y,Y)}}$ with respect to $\mathbf{Y}$ instead.
We now describe our two proposed private estimators for $\mathbf{\overline{\Omega}^{dp}(X,Y)}$ which is in the multi-party setting and $\mathbf{\overline{\Omega}^{dp}(X,X)}$ which is in the single party setting. 
 \subsubsection{Proposed private estimator for distance covariance} 
\begin{figure*}
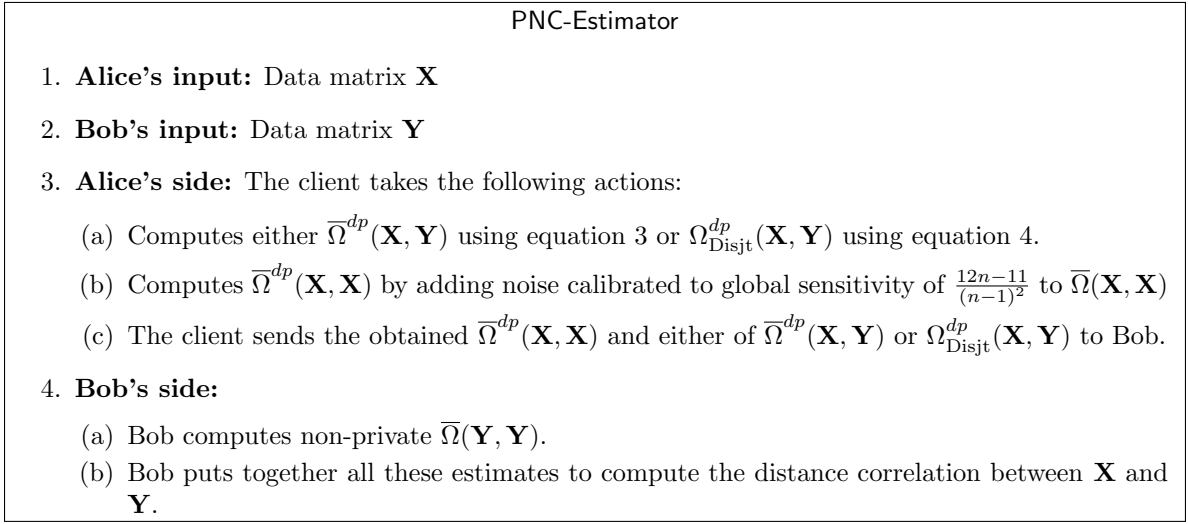

\centering
\fbox{
\begin{minipage}{6in}
\begin{center}
	$\mathsf{PNC}\mbox{-}{\mathsf{Estimator}} $ 
\end{center}
\begin{enumerate}
	\item {\bf Alice's input:} Data matrix $\mathbf{X}$
	\item {\bf Bob's input:} Data matrix $\mathbf{Y}$
	\item {\bf Alice's side:} The client takes the following actions:
	\begin{enumerate}
		\item Computes either $\overline{\Omega}^{dp}(\mathbf{X,Y})$ using equation \ref{eq:PNC} or $\Omega^{dp}_{\textrm{Disjt}}(\mathbf{X,Y})$ using equation \ref{eq:PNCDisjoint}.
		\item Computes $\overline{\Omega}^{dp}(\mathbf{X,X})$ by adding noise calibrated to global sensitivity of $\frac{12 n-11}{(n-1)^{2}}$ to $\overline{\Omega}(\mathbf{X,X})$ 
		\item The client sends the obtained $\overline{\Omega}^{dp}(\mathbf{X,X})$ and either of $\overline{\Omega}^{dp}(\mathbf{X,Y})$ or $\Omega^{dp}_{\textrm{Disjt}}(\mathbf{X,Y})$ to Bob. 
	\end{enumerate}
	\item {\bf Bob's side:} 
	\begin{enumerate}
	    \item Bob computes non-private $\overline{\Omega}(\mathbf{Y,Y})$. 
	    \item Bob puts together all these estimates to compute the distance correlation between $\mathbf{X}$ and $\mathbf{Y}$.
	\end{enumerate}
	
\end{enumerate}
\end{minipage}
}
\caption{Protocol for private estimation of non-linear correlations between data of Alice and Bob.}
\label{fig:protsub}
\end{figure*}
Our estimation starts with Alice sending $K$ differentially private random projections of sensitive data $\mathbf{X}$ to Bob that hosts the sensitive data $\mathbf{Y}$. An average of $K$ distance covariances between the random projections of $\mathbf{X}$ and the raw data $\mathbf{Y}$ is computed at Bob's premises to get $\mathbf{\overline{\Omega}^{dp}(X,Y)}$. We then perform K differentially private random projections of the data $\mathbf{X}$ by adding the necessary noise $N$ required for privacy. The next sub-section explains the process of choosing $N^X$ in order to guarantee differential privacy. Therefore, this estimator is given by \begin{equation}\label{eq:PNC}\overline{\Omega}^{dp}(\mathbf{X,Y})=\sum_{k=1}^{K}\frac{C_pC_q\overline{\Omega}_n(u_k^TX+N^x,v_k^TY)}{K}\end{equation} where $N^x = (N^x_1,\ldots,N^x_n)$. Note that in this case $\tilde{a}_{ij} = |u^T(X_i-X_j) + (N^x_{i}-N^x_{j})|$ and $\tilde{b}_{ij} = |v^T(Y_i-Y_j)|$.


 
 \subsection{Choosing $\mathbf{N^X}$ for differential privacy}
Several mechanisms have been proposed for releasing random projections of data with differential privacy including
\cite{Kenthapadi,blocki2012johnson,upadhyay2013random} and more recently \cite{xu2017dppro,gondara2020differentially}. We refer to  Appendix \ref{app:dprp} for details on one such mechanism.
\subsubsection{Avoiding sequential composition}
Applying $K$ differentially private  random projections $u_k^X$ in the Equation \ref{eq:rpdc} would lead to a differentially private estimate of nonlinear correlation. Studying the utility of this estimator would be of interest. That said, if all the $K$ random projections are applied on the entire dataset $\mathbf{X}$, it would lead to an overall privacy guarantee of $K\epsilon$ due to the sequential composition property of differential privacy. This loss of privacy budget can be avoided if the samples in the dataset (i.e the rows of $\mathbf{X, Y}$) are partitioned into $K$ disjoint subsets prior to a random projection based distance covariance measured disjointly on each subset prior to averaging them out. This would lead to an improved accounting with regards to the privacy budget, as this falls under the parallel composition property of differential privacy, thereby leading to an overall $\epsilon$-DP of the estimator if each of the individual projections was also performed with $\epsilon$-DP. This estimator is summarized below.
\paragraph{Private estimator on disjoint partitions $\Omega^{dp}_{\textrm{Disjt}}(X,Y)$:}
 Our estimator requires us to first partition the $n$ records of $\mathbf{X,Y}$ into $K$ blocks of $\{\left[X_1,Y_1\right]\,\left[X_2,Y_2\right]\ldots\left[X_K,Y_K\right]\}$ in order to avoid sequential composition. Therefore, this estimator is given by \begin{equation}\label{eq:PNCDisjoint}\Omega^{dp}_{\textrm{Disjt}}(\mathbf{X,Y})=\sum_{k=1}^{K}\frac{C_pC_q\overline{\Omega}_n(u_k^tX_k+N,v_k^tY_k)}{K}\end{equation}
 
\subsection{Utility results}
We now study the utility of the private estimator $\overline{\Omega}^{dp}(\mathbf{X,Y})$ in estimating the non-private $\overline{\Omega}(\mathbf{X,Y})$ given the effect of  noise in the private estimator. The utility can be expressed by $\overline{\Omega}^{dp}(X,Y) -\overline{\Omega}(X,Y)$.
 \begin{theorem} \textbf{(Decomposition theorem)} The difference between estimators of $\overline{\Omega}^{dp}(X,Y)$ and $\overline{\Omega}(X,Y)$ can be  expressed as
\small
 \begin{align}
 &\overline{\Omega}^{dp}(X,Y) -\overline{\Omega}(X,Y) = \overline{\Omega}(N_X,Y)\nonumber \\ &+ \sum_{n=1}^{K}\frac{4C_pC_q}{Kn
 (n-2)(n-3)}\sum_{i=1}^n\left(\sum_{l=1}^n |N_i^X-N_l^X|\sum_{l=1}^n V_k^T(Y_i-Y_l)| \right)
\end{align}
 \end{theorem}
 \begin{proof}
 The full proof for this theorem is given in Appendix \ref{appdx:Decomp}. 
 \end{proof}
Therefore, the error in estimation within this context depends on two error terms of $\hat{\Omega}(N_X,Y)$ and $\frac{4C_pC_q}{Kn
 (n-2)(n-3)}\sum_{i=1}^n\left(\sum_{l=1}^n |N_i^X-N_l^X|\sum_{l=1}^n V_k^T(Y_i-Y_l)| \right)$.  We now upper bound the second error term in the following theorem.
\begin{theorem} \textbf{(Error bound)} The error term \small $$\frac{4C_pC_q}{Kn
 (n-2)(n-3)}\sum_{i=1}^n\left(\sum_{l=1}^n |N_i^X-N_l^X|\sum_{l=1}^n V_k^T(Y_i-Y_l)| \right)$$ can be bounded
using a $t_1$ and $t_2$ that are both greater than $ 2\sigma\sqrt{(\frac{n-1}{n})(\log(2)-\frac{\log(\alpha)}{n})}$ and $C=\frac{4C_pC_q}{Kn(n-2)(n-3)}$. The concentration bound we obtain is given by \small
\begin{align}
    &\mathbb{P}\left\{C\sum_{i=1}^n  \left(\sum_{l=1}^{n}|N_i^X-N_l^X|
    \sum_{l=1}^{n}|V_k^t(Y_i-Y_l)|\right)<Cn^3t_1t_2\right\} \\&\geq (1-n(\alpha_1+\alpha_2-\alpha_1\alpha_2))\nonumber
\end{align} where $P(\sum_{l=1}^n|V_k^t(Y_i-Y_l)|\geq nt_1)\leq \alpha_1$ and $P(\sum_{l=1}^n|N^X_i-N^X_l|\geq nt_2)\leq \alpha_2$
\end{theorem}  

\begin{proof}
\textbf{Sketch of proof}
\end{proof}

 \textbf{Bounding the error term of $\overline{\Omega}(N_X,Y)$}: The error term of $\overline{\Omega}(N_X,Y)$ can be bounded using a simplification of our above result by setting $\mathbf{X}$ to be zero in the decomposition theorem and using the same strategy as the error bound in Theorem 2. That said, since the noise parameters can be released in differential privacy, a reasonable estimate of $\overline{\Omega}(N_X,Y)$ can be estimated on the premises of Bob, that holds $Y$. This analysis is provided in the Appendix A.2.

 \subsubsection{Privatizing $\overline{\Omega}(\mathbf{X,X})$}
 The problem of privatizing ${\overline{\Omega}\mathbf{(X,X)}}$ is simpler than that of privatizing ${\overline{\Omega}\mathbf{(X,Y)}}$ in the sense that in the former case, the entire data $\mathbf{X}$ is hosted on one entity as opposed to $\mathbf{X}$ and $\mathbf{Y}$ being hosted on two entities. For the private estimation of distance variance terms ${\overline{\Omega}^{dp}\mathbf{(X,X)}}$, we first use an equivalence between distance covariance and another popular dependency measure called Hilbert-Schmidt Independence Criterion (HSIC) \cite{gretton2005measuring}. We then use the global sensitivity of this equivalently obtained HSIC to privatize the distance variance. 
 We begin by defining the empirical estimate of the HSIC.
 
 Given unique positive definite kernels $k, l$ in the context of kernel methods and reproducing kernel Hilbert space (RKHS) theory in machine learning, we have the following definition for the sample estimator of HSIC. 
 \begin{definition} (\textbf{HSIC} \cite{gretton2005measuring})
 Let $Z:=\left\{\left(x_{1}, y_{1}\right), \ldots,\left(x_{m}, y_{m}\right)\right\} \subseteq \mathcal{X} \times \mathcal{Y}$ be a series of $m$ independent observations drawn from $p_{x y} .$ An estimator of HSIC, written $\operatorname{HSIC}(Z, \mathcal{F}, \mathcal{G})$, is given by
$$
\operatorname{HSIC}(Z, \mathcal{F}, \mathcal{G}):=(m-1)^{-2} \operatorname{tr} K H L H
$$
where $H, K, L \in \mathbb{R}^{m \times m}, K_{i j}:=k\left(x_{i}, x_{j}\right), L_{i j}:=l\left(y_{i}, y_{j}\right)$ and $H$ is a double-centering matrix. 
 \end{definition}
We assume $k, l$ are bounded above by 1 (e.g., the squared exponential kernel, the Matern kernel [27]). 
A classic way to calibrate the amount of noise required to achieve $\epsilon-$differential privacy \cite{dwork2014algorithmic,dwork2006differential,dwork2008differential} is to add a noise with a variance of $\frac{\Delta}{\epsilon}$ where $\Delta$ is the global sensitivity of the query. The basic definitions for differential privacy, global sensitivity are provided as a primer in the Appendix C. That said, the popular book of \cite{dwork2014algorithmic} is a wonderful resource on differential privacy. \\
\textbf{Global sensitivity of HSIC:} The global sensitivity of HSIC was derived in \cite{kusner2015inferring,kusner2016private} to be at most $\frac{12 n-11}{(n-1)^{2}}$. Specifically,
$$
\left|\widehat{H S I C}_{k, l}\left(\mathbf{x}, \mathbf{y}\right)-\widehat{H S I C}_{k, l}\left({\mathbf{x}}^{\prime}, {\mathbf{y}}^{\prime}\right)\right| \leq \frac{12 n-11}{(n-1)^{2}}
$$ for all neighboring \cite{dwork2014algorithmic} datasets. 

\begin{figure*}
\begin{subfigure}{0.3\textwidth}
   \includegraphics[width=\linewidth]{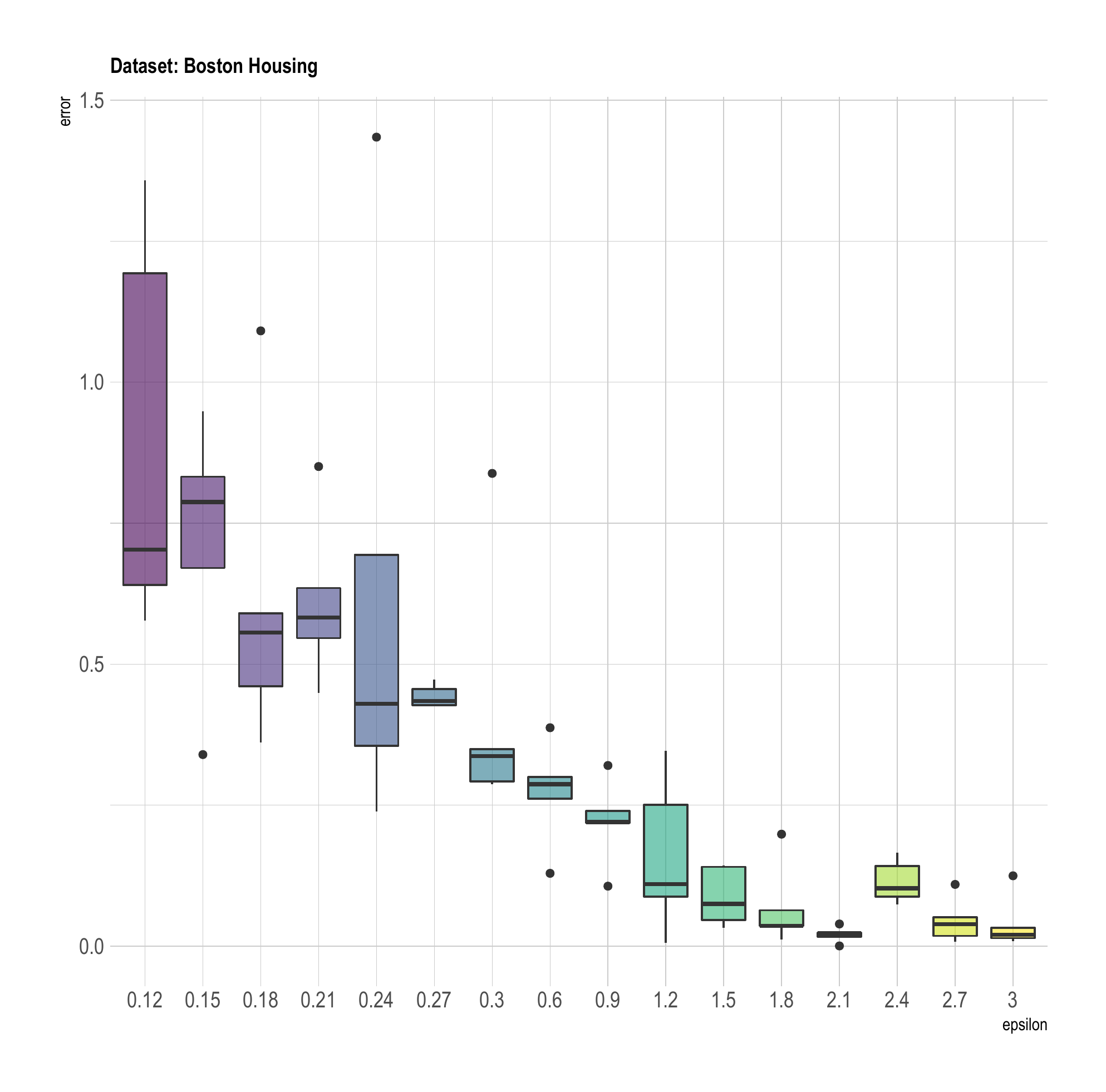}
   \centering
   \caption{Privacy-utility tradeoff on the Boston Housing dataset.} \label{fig:x_a}
\end{subfigure}
\hspace*{\fill}
\begin{subfigure}{0.3\textwidth}
   \includegraphics[width=\linewidth]{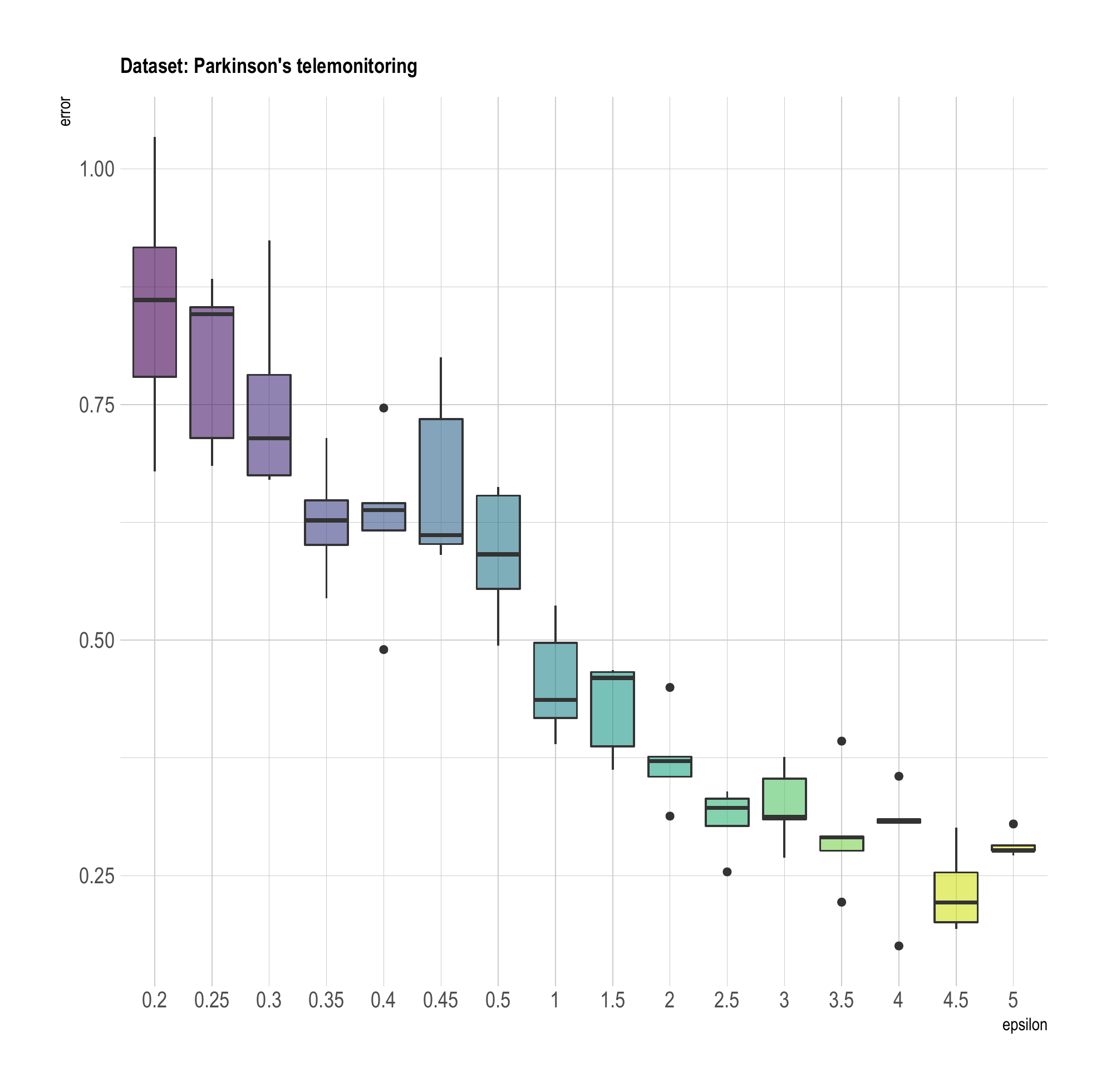}
   \caption{Privacy-utility tradeoff on the Parkinson's tele-monitoring dataset.} \label{fig:x_b}
\end{subfigure}
\hspace*{\fill}
\begin{subfigure}{0.3\textwidth}
   \includegraphics[width=\linewidth]{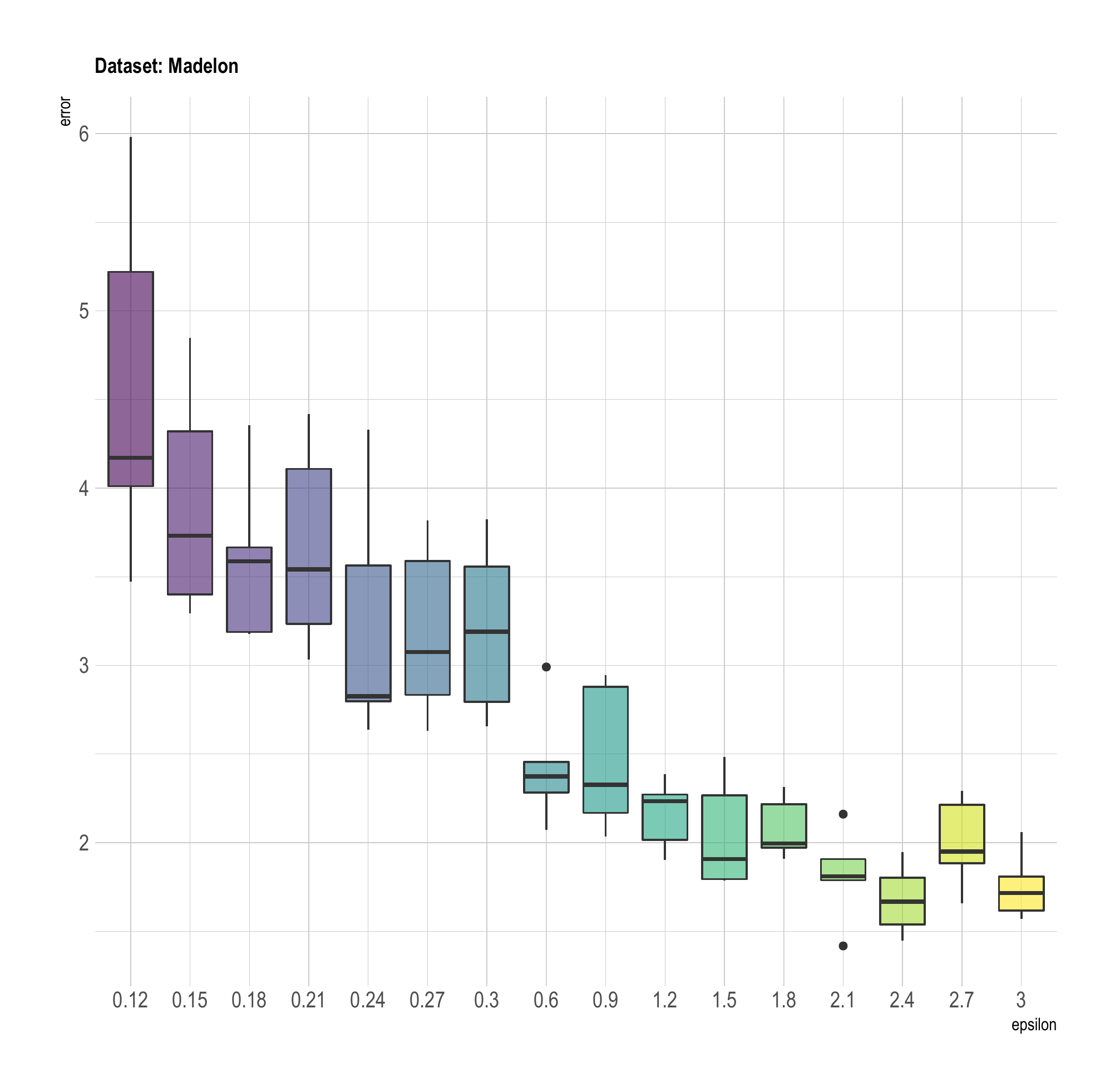}
   \caption{Privacy-utility tradeoff on the Madelon dataset.} \label{fig:x_c}
\end{subfigure}

\bigskip
\begin{subfigure}{0.3\textwidth}
   \includegraphics[width=\linewidth]{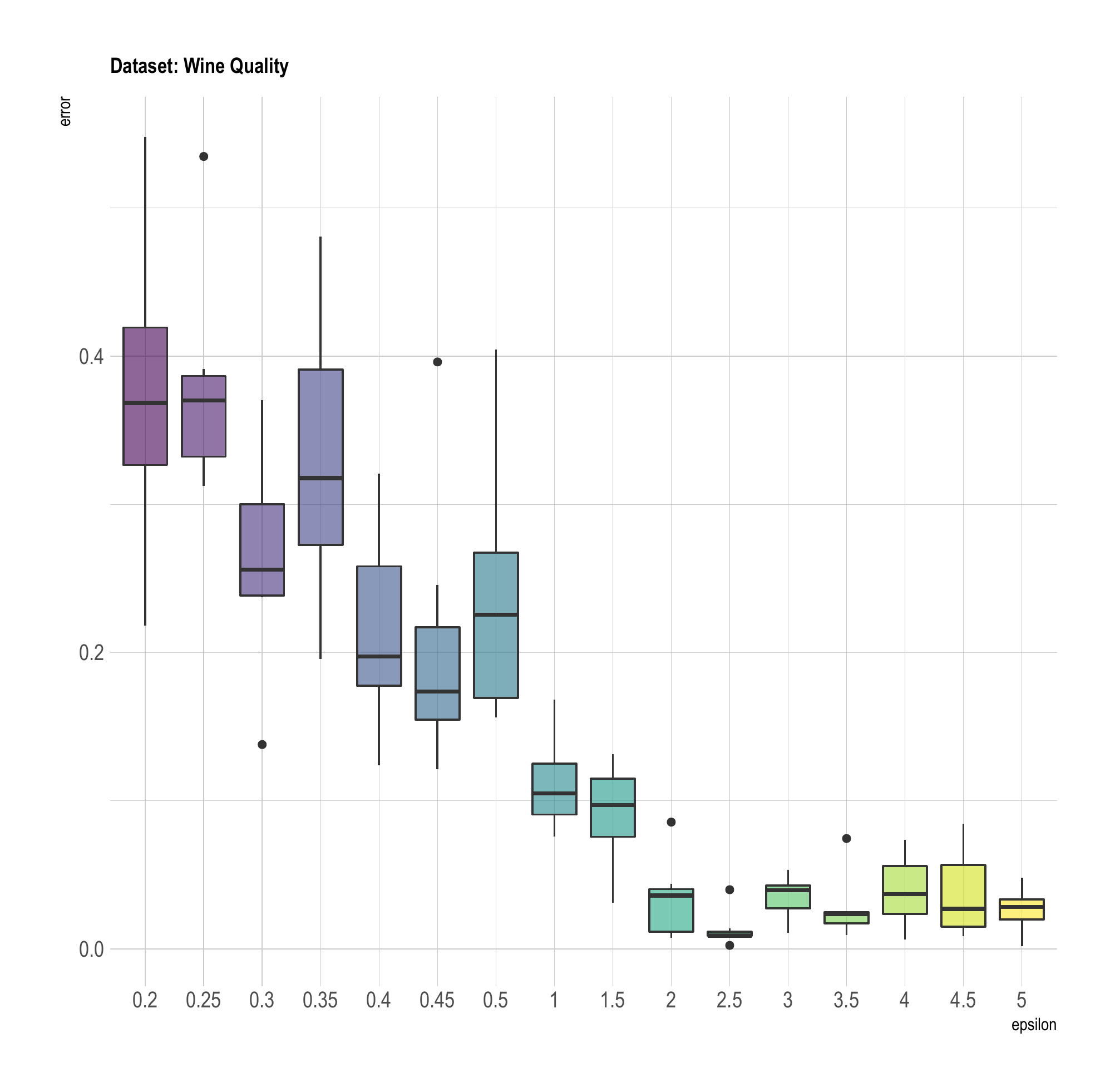}
   \caption{Privacy-utility tradeoff on the wine quality dataset.} \label{fig:x_d}
\end{subfigure}
\hspace*{\fill}
\begin{subfigure}{0.3\textwidth}
   \includegraphics[width=\linewidth]{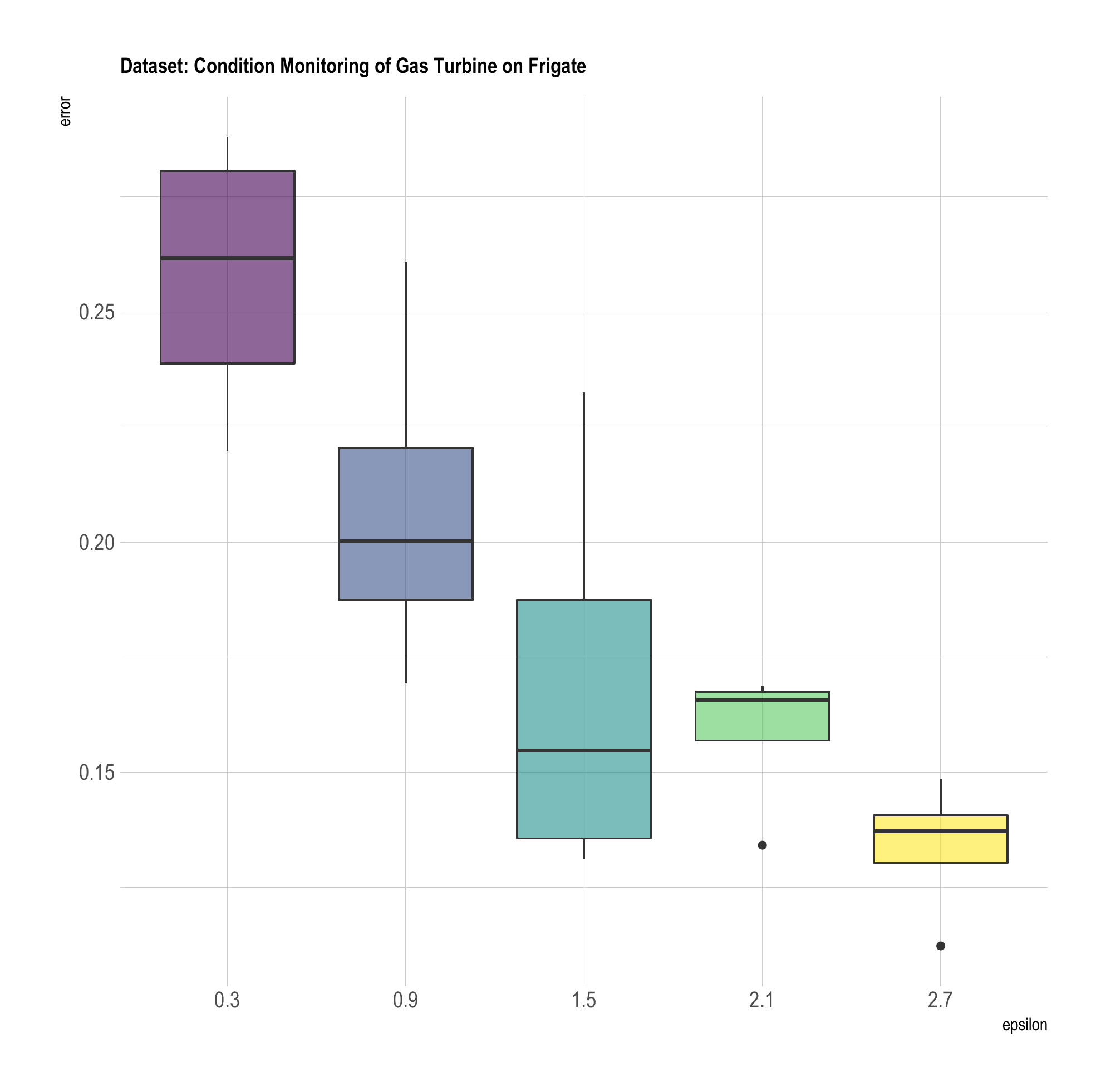}
   \caption{Privacy-utility tradeoff on the condition monitoring of gas turbine on a frigate dataset.} \label{fig:x_e}
\end{subfigure}
\hspace*{\fill}
\begin{subfigure}{0.3\textwidth}
   \includegraphics[width=\linewidth]{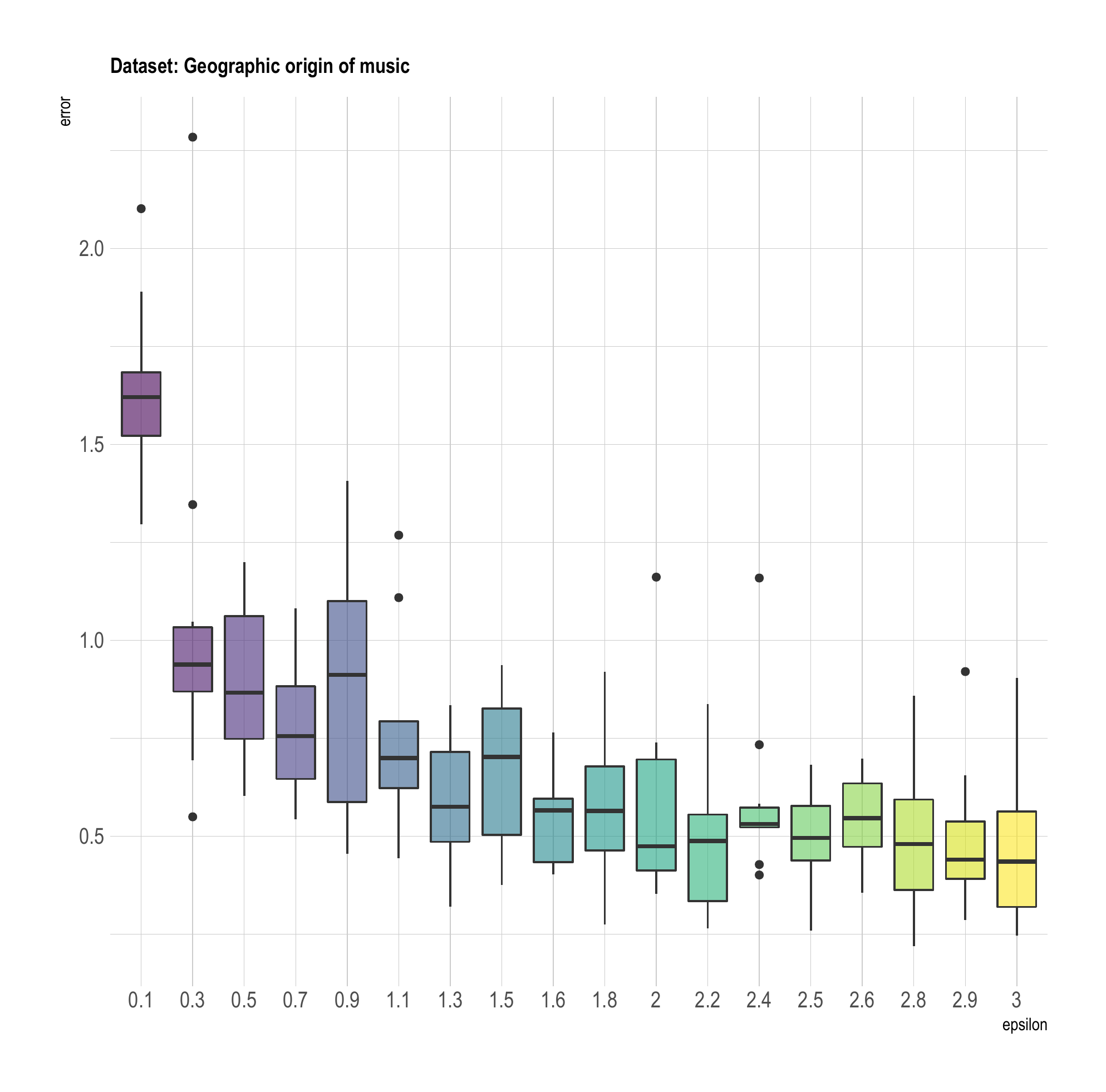}
   \caption{Privacy-utility tradeoff on the geographic origin of music dataset.} \label{fig:x_f}
\end{subfigure}

\bigskip
\hspace*{\fill}
\begin{subfigure}{0.3\textwidth}
   \includegraphics[width=\linewidth]{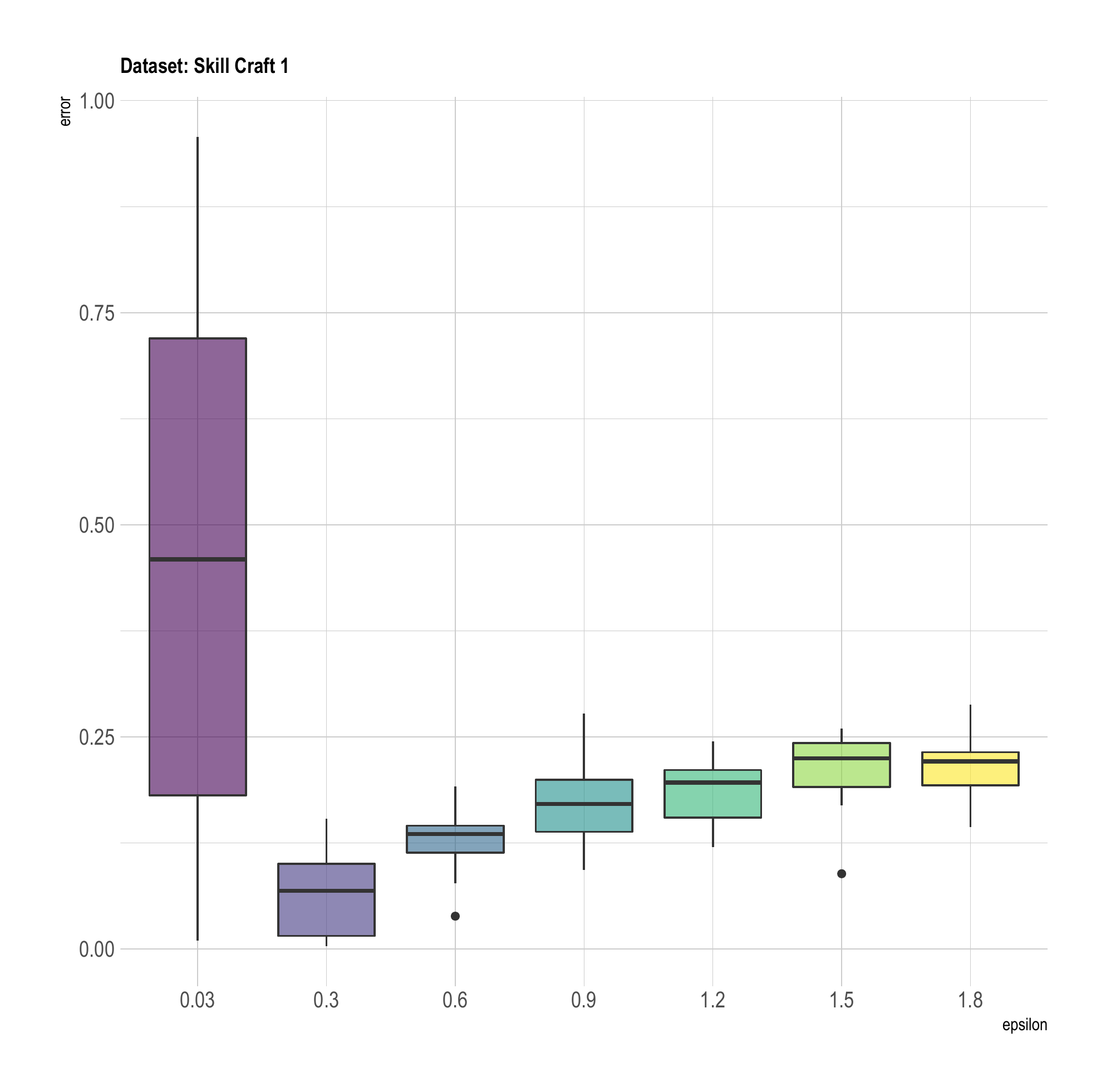}
   \caption{Skill Craft Dataset: Effect of $\#$ of repeated accesses to data being 3 in the no-partition estimator of $\overline{\Omega}^{dp}(X,Y)$ } \label{fig:x_g}
\end{subfigure}%
\hspace*{0.05\textwidth}%
\begin{subfigure}{0.3\textwidth}
   \includegraphics[width=\linewidth]{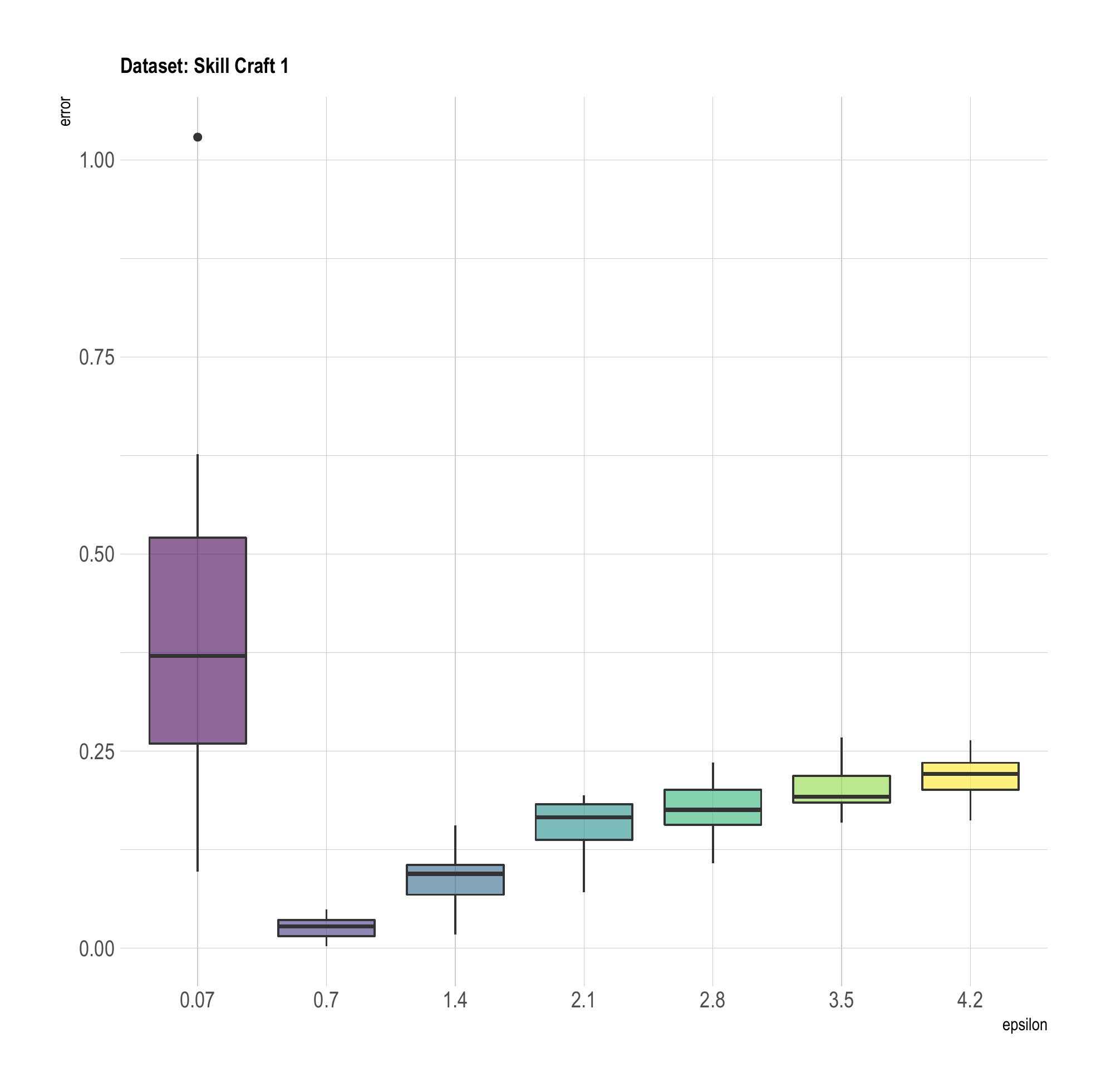}
   \caption{Skill Craft Dataset: Effect of $\#$ of repeated accesses to data being 7 in the no-partition estimator of $\overline{\Omega}^{dp}(X,Y)$. As can be seen the variance reduced im comparison to Figure (g) while the loss in privacy budget was more drastic due to the repeated data accesses. Thereby, we make a case for $\overline{\Omega}^{dp}_{\textrm{Disjt}}(X,Y)$ which has disjoint partitions and no repeated access to raw data. \label{fig:x_h}}
\end{subfigure}
\hspace*{\fill}

\caption{Privacy-utility trade-offs on several datasets is studied in (a) to (f) and the effect of partitioning or not is studied on the skill craft dataset in (g) to (h).}
\end{figure*}
The following equivalence was shown between distance covariance and HSIC. As we have the global sensitivity for HSIC and since we have the following equivalence, we can use this global sensitivity to privatize $\overline{\Omega}(\mathbf{X,X})$, which belongs to the one-party setting of privately releasing distance covariance.
 \subsection{Equivalence map between Distance Covariance and HSIC}
\begin{definition}
Bijective induced kernel 
\end{definition} 
Given sample data $\left\{x_{i}, i=1, \ldots, n\right\} .$ For any metric $d(\cdot, \cdot)$, we define its bijective induced kernel \cite{shen2019exact,shen2021exact} as
$$
\hat{k}_{d}\left(x_{i}, x_{j}\right)=\max _{s, t \in[n]}\left(d\left(x_{s}, x_{t}\right)\right)-d\left(x_{i}, x_{j}\right)
$$
For any kernel $k(\cdot, \cdot)$, we define the induced metric as
$$
\hat{d}_{k}\left(x_{i}, x_{j}\right)=\max _{s, t \in[n]}\left(k\left(x_{s}, x_{t}\right)\right)-k\left(x_{i}, x_{j}\right)
$$
The subscripts $s, t \in[n]$ is a shorthand for $s=1, \ldots, n$ and $t=1, \ldots, n$. 

Other alternative definitions for this bijection are also given in \cite{shen2021exact}. These are wieldy to use depending on the problem at hand.
$$
\begin{aligned}
&\hat{d}_{k}\left(x_{i}, x_{j}\right)=1-k\left(x_{i}, x_{j}\right) / \max _{s, t \in[n]}\left(k\left(x_{s}, x_{t}\right)\right) \\
&\hat{k}_{d}\left(x_{i}, x_{j}\right)=1-d\left(x_{i}, x_{j}\right) / \max _{s, t \in[n]}\left(d\left(x_{s}, x_{t}\right)\right)
\end{aligned}
$$
This is an equivalent definition up-to scaling by the maximum elements, and can be succinctly expressed in a matrix form:
$$
\begin{aligned}
&\hat{\mathbf{D}}_{\mathbf{K}}=\mathbf{J}-\mathbf{K} / \max (\mathbf{K}) \\
&\hat{\mathbf{K}}_{\mathbf{D}}=\mathbf{J}-\mathbf{D} / \max (\mathbf{D})
\end{aligned}
$$

\begin{theorem} \cite{shen2021exact}
Suppose distance covariance (DCOV) uses a given metric $d(\cdot, \cdot)$, and the Hilbert Schmidt independence criterion HSIC uses the bijective induced kernel $\hat{k}_{d}(\cdot, \cdot)$.
Given any sample data $(\mathbf{X}, \mathbf{Y})$, it holds that
$$
\begin{aligned}
&\operatorname{DCOV}_{n}(\mathbf{X}, \mathbf{Y})=\operatorname{HSIC}_{n}(\mathbf{X}, \mathbf{Y}) 
\end{aligned}
$$
where the remainder term $O\left(\frac{1}{n^{2}}\right)$ is invariant to permutation.
\end{theorem}

\section{Computational complexity}
Fast estimators of distance correlation requires $\mathcal{O}(nlogn)$ \cite{chaudhuri2019fast,huo2016fast} computational complexity for univariate and $\mathcal{O}(nKlog n)$ complexity \cite{huang2017statistically} for multivariate settings with $\mathcal{O}(\max(n, K))$ memory, where
$K$ is the number of random projections required as part of the estimation. In addition to being differentiable and easily computable with a closed-form, it requires no other tuning of parameters and is self-contained unlike HSIC or other dependency measures such as Maximum Mean Discrepancy (MMD) or Kernel Traget Alignment (KTA) that depend on a choice of separate kernels for features as well as labels along with their respective tuning parameters. 

\section{Experiments}

In this section, we first introduce the datasets, experimental methodology and metrics of evaluation. We study the performance of our method in privately estimating distance correlation across a variety of datasets. We study the privacy-utility trade-off curves for varying values of privacy parameter $\epsilon$. We study the effect of $\epsilon$ on the $l_1$ error of estimation. 

\subsection{Datasets}
\begin{enumerate}
    \item \textbf{Condition Based Maintenance of Naval Propulsion Plants} This dataset of 11934 samples has been generated from a sophisticated simulator of a Gas Turbines (GT), mounted on a Frigate. The goal is to predict performance decay of this equipment based on 15 features.
    \item \textbf{Parkinsons Telemonitoring Data Set} This dataset of 5875 instances and 24 features was created in collaboration with 10 medical centers in the US and Intel Corporation who developed the tele-monitoring device to record the speech signals. The goal is to predict the clinician's Parkinson's disease symptom score on the UPDRS scale.
   
    \item \textbf{Wine Quality Data} This is a dataset of 4898 samples and 11 features. The goal is to predict the red-wine quality using the input features. Due to privacy and logistical issues, only some variables are made available.
    \item \textbf{Boston Housing} This small dataset of 506 samples and 18 features contains information collected by the U.S Census Service concerning housing in the area of Boston Mass. It has classically been used extensively throughout the literature to benchmark algorithms.
    
    \item \textbf{Madelon} This dataset is one of five datasets used in the NIPS 2003 feature selection challenge and has 4400 samples of 500 features. It is an artificial dataset containing a number of distractor feature called 'probes' having no predictive power. The order of the features and patterns were randomized. Only 20 out of the 500 are real features with predictive value and 480 are the distracting features/probes.
    
    \item \textbf{Geographic origin of music} This dataset was built from a personal collection of 1059 tracks covering 33 countries/areas. The goal is to predict the geographical origin of each music file based on the audio features extracted from each wave file. This dataset has 3395 samples and 19 features.
    
    \item \textbf{Skill craft 1 Master Table} This is a gaming related dataset where the goal is to predict the league (position) of any player of skill craft 1 based on features that measure his interactions with the game. This dataset has 1059 samples and 116 features.
    
    \item \textbf{Seoul bike sharing demand} The goal is to predict the bike count required at each hour in order to plan and maintain a stable supply of rental bikes.This dataset has 8760 samples and 13 features. The results on this dataset are provided in Appendix E.
\end{enumerate}
\vspace{-6mm}
\section{Methodology} We split the feature sets for each of these datasets into two. One to represent $\mathbf{X}$ and other to represent $\mathbf{Y}$. Note that distance covariance and distance correlation can be computed between $\mathbf{X}$ and $\mathbf{Y}$ that can potentially lie in different dimensions as long as their sample sizes match up to be the same. The reason for this is that the estimators for distance covariance and distance correlation solely depend on the distance matrices of these samples. Since distance matrices are square and symmetric and since the sample sizes match up, thereby the operations commute fine even if the original data samples lied in different dimensions. We then apply our proposed protocol for computing the private distance correlations as shown in Figure 2. We use both the proposed estimators of $\overline{\Omega}^{dp}(\mathbf{X,Y})$ or $\Omega^{dp}_{\textrm{Disjt}}(\mathbf{X,Y})$ along with $\overline{\Omega}^{dp}(\mathbf{X,X})$ on each of our considered datasets. We compute the $l_1$ error between the true sample distance correlation (non-private) and the proposed private estimators for each choice of privacy parameter $\epsilon$. We plot these $\epsilon$ Vs. utility curves for each ouf the considered dataset. In order to understand the sense of variability around these utility estimates, we repeat the experiment multiple times and report the variances observed in these plots from Figures 3 (a) to 3(f). In most cases, we observe a substantially good utility in measuring sample distance correlation at reasonably low levels of $\epsilon$. Furthermore, we studied the effect of using the proposed disjoint private estimator $\Omega^{dp}_{\textrm{Disjt}}(\mathbf{X,Y})$  Vs. using the repeated data access private estimator $\overline{\Omega}^{dp}(\mathbf{X,Y})$ in the following subsection. The privacy budgets in all the plots have been calibrated according to the rules of their corresponding composition properties for differential privacy\cite{dwork2014algorithmic}. The effect of number of partitions is also studied below. 
\vspace{-3.5mm}
\subsection{Effect of disjoint Vs. repeated partitions}
We study the effect of using the proposed estimator on disjoint partitions $\overline{\Omega}^{dp}_{\textrm{Disjt}}(X,Y)$ as opposed to the entire dataset being accessed for each $k$ as in the proposed estimator of $\overline{\Omega}^{dp}(X,Y)$. Using disjoint partitions is beneficial to conserving the privacy-budget via parallel composition property while on the other side it reduces the number of samples the estimator for any $k$ would have access to. This potentially has an effect on the utility or variance of the estimator. This can been empirically observed as shown in the variances of Figure \ref{fig:nopart} mostly being lesser than that of Figure \ref{fig:disjt} .

\begin{figure}
    \centering
    \includegraphics[scale=0.18]{images/skillCraft.pdf}
    \caption{$\overline{\Omega}^{dp}(X,Y)$ - No partitions}
    \label{fig:nopart}
    \includegraphics[scale=0.18]{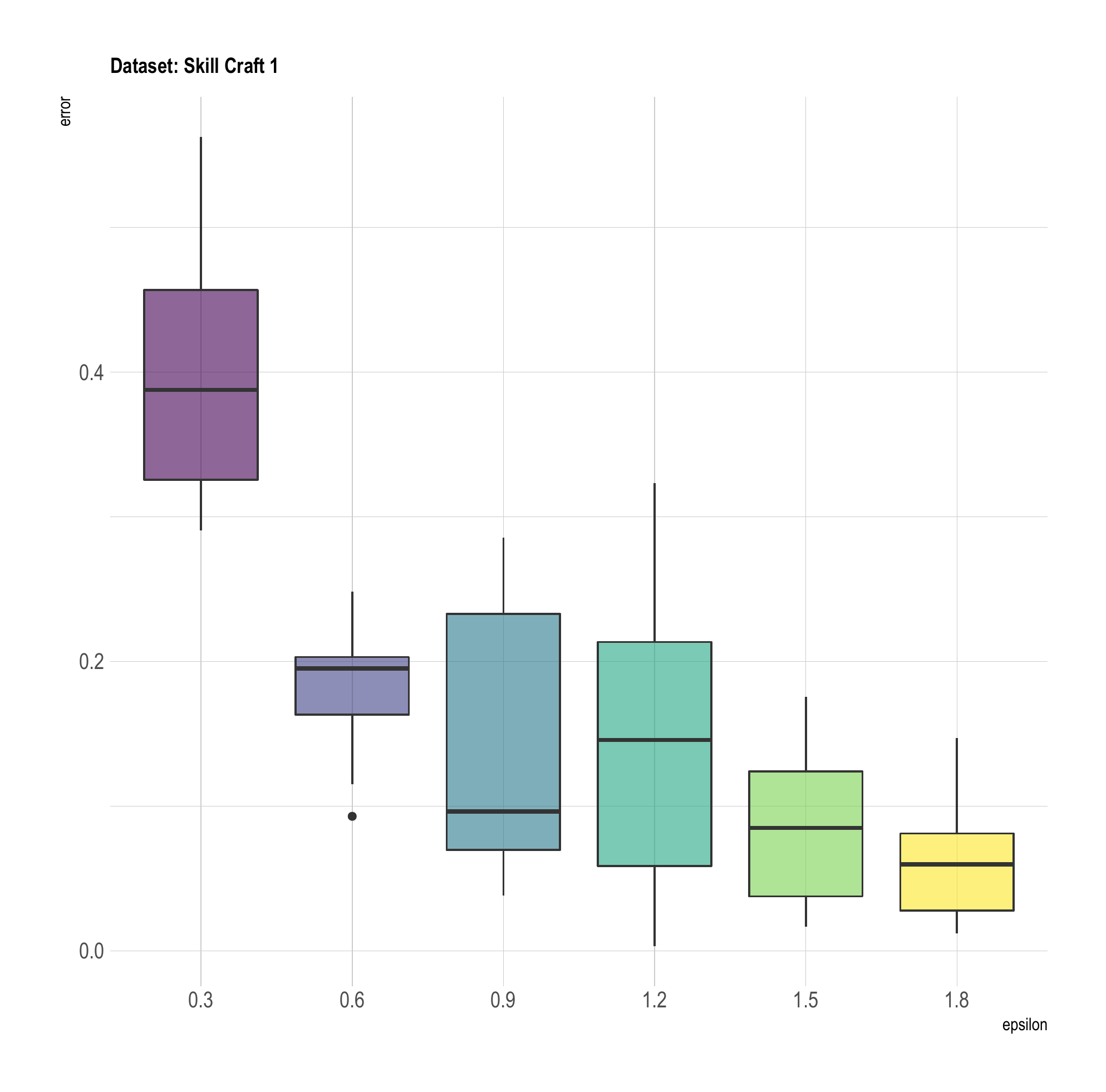}
    \caption{$\overline{\Omega}^{dp}_{\textrm{Disjt}}(X,Y)$ - Disjoint partitions }
    \label{fig:disjt}
\end{figure}
\vspace{-5mm}
\subsection{Effect of number of partitions} The choice of number of partitions has a trade-off between a.) reducing the variance of the estimates and b.) losing the privacy-budget when it comes to the estimator with repeat accesses to data. This is as opposed to the estimator with disjoint partitions, where the trade-off is between reducing the variance of the estimator and loss of utility for the estimator, due to smaller set of samples being accessible within any partition.

\vspace{-4mm}
\section{Conclusion} We provided a method to privately compute distance correlation between data hosted on two different parties. We provided theoretical and empirical justification of our method. As measurement of nonlinear correlations is a fundamental topic of interest to theoreticians and practitioners alike, we thereby hope that this work brings this problem of private nonlinear correlation estimation to a further mainstream consideration.
\vspace{-7mm}
\section{Future work} \vspace{-3mm}Distance correlation is a special case of a broader concept of energy statistics \cite{szekely2013energy,rizzo2016energy,szekely2017energy}. Thereby solutions for its private estimation open up a door for investigating multi-party private solutions for downstream problems that depend on distance correlation such as multi-party private independence testing, multi-party private feature screening and private multi-party private causal inference.

\onecolumn
\appendix
\section{Proof of decomposition theorem}
\label{appdx:Decomp}
\small
 \begin{proof}
 
 \begin{align}
 \overline{\Omega}^{dp}(X+N_X,Y) &= \sum_{n=1}^{K}\frac{C_pC_q \Omega_n(u_{K}^{T}X + N_X, v_{K}Y)}{K}\nonumber =\sum_{n=1}^{K} \frac{C_pC_q}{K} \Bigl\{ \frac{1}{n(n-3)}\sum_{i \neq j}\tilde{a}_{ij}\tilde{b}_{ij} -\frac{2}{n(n-2)(n-3)}\sum \tilde{a_{i\cdot}}\tilde{b_{i\cdot}} +\\& \frac{a_{..}b_{..}}{n(n-1)(n-2)(n-3)}\Bigr\}\nonumber =\sum_{n=1}^{K} \frac{C_pC_q}{K} \biggl\{\nonumber \frac{1}{n(n-3)}\sum_{i\neq j}| u_K^T(X_i-X_j) + N_i^X - N_j^X | |V_K^T(Y_i-Y_j)|\\\nonumber&- \frac{2}{n(n-2)(n-3)}\sum_{i=1}^{n} (\sum_{j=1}^n | u_K^T(X_i-X_j) + N_i^X - N_j^X | \sum_{j=1}^n  |V_K^T(Y_i-Y_j)|) \nonumber\\ &+\frac{(\sum_{i,j}^{n}| u_K^T(X_i-X_j) + N_i^X - N_j^X | \sum_{k,l}^{n}|V_K^T(Y_k-Y_l)| )}{n(n-1)(n-2)(n-3)}\biggr\}\nonumber  \\ &\leq \sum_{n=1}^{K} \frac{C_pC_q}{K} \Bigl\{ \frac{1}{n(n-3)}\sum_{i\neq j}| u_K^T(X_i-X_j)||V_K^T(Y_i-Y_j)| + \frac{1}{n(n-3)}\sum_{i\neq j} |(N_i^X-N_j^X)||V_K^T(Y_i-Y_j)|\nonumber\\&-\frac{2}{n(n-2)(n-3)}\sum_{i=1}^n\biggl(\sum_{j=1}^n | u_K^T(X_i-X_j)|\sum_{j=1}^n V_K^T(Y_i-Y_j)|\biggr)\nonumber \\&+\frac{2}{n(n-2)(n-3)}\sum_{i=1}^n\biggl(\sum_{j=1}^n | N_i^X-N_j^X|\sum_{j=1}^n V_K^T(Y_i-Y_j)| \biggr) \nonumber \\ &+ \frac{1}{n(n-1)(n-2)(n-3)}\biggl(\sum_{i,j}^{n}| u_K^T(X_i-X_j)|\sum_{k,l}^{n}|V_K^T(Y_k-Y_l)|\nonumber +\sum_{i,j}^{n} |N_i^X - N_j^X |\sum_{k,l}^{n}|V_K^T(Y_k-Y_l)| \biggr)\Bigr\} \nonumber \\ & = \sum_{n=1}^{K} \frac{C_pC_q}{K} \Bigl\{ \frac{1}{n(n-3)}\sum_{i\neq j}| u_K^T(X_i-X_j)||V_K^T(Y_i-Y_j)|
-\frac{2}{n(n-2)(n-3)}\sum_{i=1}^n\biggl(\sum_{j=1}^n | u_K^T(X_i-X_j)|\sum_{j=1}^n V_K^T(Y_i-Y_j)|\biggr)\nonumber \\&+\frac{1}{n(n-1)(n-2)(n-3)}\biggl(\sum_{i,j}^{n}| u_K^T(X_i-X_j)|\sum_{k,l}^{n}|V_K^T(Y_k-Y_l)|\biggr)\nonumber \\&-\frac{2}{n(n-2)(n-3)}\sum_{i=1}^n\biggl(\sum_{j=1}^n | N_i^X-N_j^X|\sum_{j=1}^n V_K^T(Y_i-Y_j)| \biggr)+ \frac{1}{n(n-3)}\sum_{i\neq j} |(N_i^X-N_j^X)||V_K^T(Y_i-Y_j)|\nonumber \\ &+ \frac{1}{n(n-1)(n-2)(n-3)}\biggl(\sum_{i,j}^{n} |N_i^X - N_j^X |\sum_{k,l}^{n}|V_K^T(Y_k-Y_l)| \biggr)\Bigr\}\nonumber \\& + \sum_{n=1}^K \frac{4C_pC_q}{Kn(n-2)(n-3)}\sum_{i=1}^n\biggl(\sum_{j=1}^n | N_i^X-N_j^X|\sum_{j=1}^n V_K^T(Y_i-Y_j)| \biggr)\nonumber \\& = \overline{\Omega}(X,Y) + \overline{\Omega}(N^X,Y) + \sum_{n=1}^K \frac{4C_pC_q}{Kn(n-2)(n-3)}\sum_{i=1}^n\biggl(\sum_{j=1}^n | N_i^X-N_j^X|\sum_{j=1}^n V_K^T(Y_i-Y_j)| \biggr)
\end{align}\end{proof}
We have used the inequalities $|a+b|\leq |a|+|b|$ and $|a+b|\geq |a|-|b|$ for two vectors $a$ and $b$ where $|a|$ refers to the modulus of vector $a$. For any fixed $n$, $\hat{\Omega}(X,Y)$ is denoted by $\Omega_n(X,Y)$. Here $\overline{\Omega}(N^X,Y)$ is defined by 
\begin{lemma}

For $a,b$ and $c$ $\in \mathbb{R}^+$, given that $p\left(a_{i}<c\right)>b$ for $i=1,2, \ldots, n$ we will show that $P\left(\sum a_{i}<n c\right) \geqslant 1-n(1-b)$.\\
\end{lemma}
\begin{proof}

$p\left(\sum a_{i}<n c\right) \geqslant P\left(a_{1}<c \cap a_{2}<c \ldots .\right.$
$\left.\cap a_{n}<c\right)$
$=1-P\left(a_{1} \geqslant c \cup a_{2} \geqslant c \cup a_{3} \geqslant C\right.$
$\left. \ldots \quad \cup a_{n} \geqslant c\right)$
$\geqslant 1-p\left(a_{1} \geqslant c\right)-p\left(a_{2} \geqslant c\right)$
$\ldots-p\left(a_{n}> c\right)$
$\geq 1-n(1-b)$
\end{proof}

 \subsection{Proof for the bound on error term}
  \begin{proof}
Now as the $N_i^X$'s are drawn from a mean zero Gaussian distribution where $\sigma^2 \geq w_2(P) \frac{\sqrt{2 ln(1/2\delta) + \epsilon}}{\epsilon}$. Assuming $N_i^X$ and $N_l^X$ are sampled independently for any $l\neq i$, if $s_l = N_i^X-N_l^X$, the distribution of $s_l$ terms will be therefore $N(0,2\sigma^2)$. Now let $\tau_i = \sum_{l=1 }^n |N_i^X-N_l^X|=\sum_{l=1, l\neq i}^n |N_i^X-N_l^X|=\sum_{l=1, l\neq i}^n |s_l|$ It is critical to note that the $s_l$'s are not independent anymore. Therefore, we consider $\tau_i$'s instead as for the expectation $\mathbb{E}(\tau_i)$, the correlation between $s_l$'s would not matter anymore. Now, $\mathbb{E}(\tau_i) = (n-1) \mathbb{E}(s_1)$. This brings us into the regime of half-normal or folded normal distributions. If $X \sim N(0,\sigma^2), Y=|X|$, then we know that $\mathbb{E}[Y] = \sigma \sqrt{\frac{2}{\pi}}$ and $Var(Y) =\sigma^2(1-\frac{2}{\pi})$. In our case $\mathbb{E}[s_1] = \sigma \sqrt{2}\times \sqrt{\frac{2}{\pi}} = \frac{2 \sigma}{\sqrt{\pi}}$ and  $Var(s_1) = 2\sigma^2 (1- \frac{2}{\pi})$. We now need a concentration bound for $\sum_{l=1, l\neq i}^{n} |N_i^X-N_l^X|$.
\small
\begin{align}
&\sum_{l=1,l\neq i}^{n}|N_i^X-N_l^X| \leq \nonumber   (n-1)|N_i^X| + \sum_{l=1,l\neq i}^{n}|N_l^X|  \nonumber  
\end{align}
The moment generating function of a standard Gaussian random variable $X$ is
$\mathbb{E}[e^{tX}] = e^\frac{\sigma^2 t^2}{2}$ for all $t \in \mathbb{R}$. We would want to find a concentration bound on $|(n-1)N_i^X| + \sum_{l=1,l\neq i}^{l=n} |N_i^X|$. We denote $(n-1)N_i^X$ as $N_i^\prime$. Note that $(n-1)N_i^X$ and $N_l^X$ where $l\neq i$ are all i.i.d. The required concentration bound can be written in terms of the moment generating function as
\begin{center}
   
\small
\begin{align}
    P\left\{ \frac{1}{n}\left(|(n-1)N_i^X| + \sum_{l=1;l\neq i}^{l=n} |N_i^X|\right )\geq t\right\}\nonumber   &= \nonumber
    P\left\{ \exp{\left(\frac{\lambda}{n}(|(n-1)N_i^X| + \sum_{l\neq i}^{l=n} |N_i^X|)\right)\geq e^{\lambda t}}\right\} \\ \nonumber
    &\leq e^{-\lambda t} \mathbb{E}\left[\exp{\left(\frac{\lambda}{n}(|(n-1)N_i^X| + \sum_{l=1,l\neq i}^{l=n} |N_i^X|)\right)}\right] \\\nonumber  
    &= e^{-\lambda t} \mathbb{E}\left[\left(\prod_{k=1;k\neq i}^{n}e^{\frac{\lambda}{n}|N_k^X|}\right)e^{\frac{\lambda}{n} |(n-1)N_i^X|}\right]
    \\\nonumber  
    &=e^{-\lambda t} \prod_{k=1;k\neq i}^{n}\mathbb{E} \left[e^{\frac{\lambda}{n}|N_k^{X}|}\right]\mathbb{E} \left[e^{\frac{\lambda}{n} |(n-1)N_i^{X}|}\right] \text{(as they are independent)} \label{indepEqn}
\end{align}
\end{center}
Now based on the properties of moment generating functions we have,\begin{align}
\mathbb{E}\left[e^{t|X_k|}\right] &\leq \mathbb{E}\left[e^{t|X_k|} + e^{-t|X_k|}\right]\\ \nonumber
&=\mathbb{E}[e^{tX_k} + e^{-tX_k}]\\\nonumber
&=\mathbb{E}[e^{tX_k}] + \mathbb{E}[e^{tX_k}]\\\nonumber
&=2\mathbb{E}[e^{tX_k}] \\\nonumber
&=2e^{\frac{\sigma^2t^2}{2}} \text{as m.g.f for Gaussians is } m_X(t)=e^{\mu t + \frac{\sigma^2t^2}{2}}
\end{align}
 
We have that $\mathbb{E}[e^{\frac{\lambda}{n}|N_{k}^{X}|}] \leq 2e^{\frac{\lambda^2\sigma^2}{2n^2}}$ and $\mathbb{E}[e^{\frac{\lambda}{n}|(n-1)N_i^{X}|}] \leq 2e^{\frac{\lambda ^2}{2n^2}(n-1)\sigma ^2}$.

From eqn. \ref{indepEqn} the r.h.s of the inequality is $e^{-\lambda t}2^{n}e^\frac{\lambda^2 \sigma^2(n-1)}{n^2}$

Therefore $$\mathbb{P}\{expression \geq nt \} \leq e^{-\lambda t}2^{n}e^\frac{\lambda^2 \sigma^2(n-1)}{n^2}$$

As we want to minimize the upper bound,
$\therefore \frac{\partial}{\partial l}(-\lambda t + \frac{\lambda^2 \sigma^2(n-1)}{n^2} = 0)$
or, $-t + 2\lambda \frac{(n-1)\sigma^2}{n^2} = 0$. So we have, $\lambda = \frac{n^2 t}{2\sigma^2(n-1)} $.

Therefore, $ \mathbb{P}\{expression \geq nt \}  \leq e^{-\frac{n^2t^2}{4\sigma^2(n-1)}}2^n$ and in the case, where the probability needs to kept under a chosen level $\alpha$, then we have $ \left(\frac{2}{e^{\frac{nt^2}{4\sigma^2(n-1)}}}\right)^n < \alpha $. Upon applying $\log$ on both sides we have
$ n \log\left(\frac{2}{e^{\frac{nt^2}{4\sigma^2(n-1)}}}\right)<\log(\alpha)$ and upon rearranging the terms we have,
$$t > 2\sigma\sqrt{(\frac{n-1}{n})(\log(2)-\frac{\log(\alpha)}{n})}$$

Let us denote $expression = |(n-1)N_i^{X}| + \sum_{l=1,l\neq i}^{l=n}| N_l^X |$ for brevity. We know that $$P(expression \geq nt) \leq e^{-\frac{n^2t^2}{4\sigma^2(n-1)}} \times 2^n$$ Note that $\sum_{l=1}^{n}|N_i^X-N_l^X| \leq expression$ . Therefore we have, $$P(\sum_{l=1}^{n}|N_i^x-N_l^X| \geq nt) \leq e^{-\frac{n^2t^2}{4\sigma^2(n-1)}} \times 2^n$$ We now work with the next expression $\sum_{l=1}^n |V_k^t(Y_i-Y_l) |$. Note that each entry of $V_k^T$ is sampled independently from a Gaussian as part of ensuring differential privacy. As $Y_i$'s are one-hot encoded, therefore $V_k^TY_i$ and $V_k^TY_l$ are two independently sampled entries from the vector $V_k^T$. Assuming, $V_k^tY_i$ are from distribution $N(0,\sigma_1^2)$ and $N^X_i$s are from distribution $N(0,\sigma_2^2)$, we can write that

$$P(\sum_{l=1}^n|V_k^T(Y_i-Y_l)|\geq nt_1))\leq e^{\frac{-n^2t_1^2}{4(n-1)\sigma_1^2}}\times 2^n$$ and $$P(\sum_{l=1}^n|N^X_i-N^X_l)|\geq nt_2))\leq e^{\frac{-n^2t_2^2}{4(n-1)\sigma_2^2}}\times 2^n$$. Denoting the rhs of these inequalities as $\alpha_1$ and $\alpha_2$ respectively, we can write $$P(\sum_{l=1}^n|V_k^T(Y_i-Y_l)|< nt_1))\geq  1-\alpha_1$$ and $$P(\sum_{l=1}^n|N^X_i-N^X_l|< nt_2))\geq 1-\alpha_2$$ 
Note that $P(\sum_{l=1}^n|V_k^T(Y_i-Y_l)|\sum_{l=1}^n|N^X_i-N^X_l|<n^2t_1t_2)\geq P(\sum_{l=1}^n|V_k^T(Y_i-Y_l)|<nt_1)P(\sum_{l=1}^n|N^X_i-N^X_l|<nt_2)$ since these two events are clearly independent and for independent $a,b$, $P(ab<c^2)>P(a<c)P(b<c)$ holds.

Therefore, $P(\sum_{l=1}^n|V_k^T(Y_i-Y_l)|\sum_{l=1}^n|N^X_i-N^X_l|<n^2t_1t_2)\geq (1-\alpha_1)(1-\alpha_2)$.
Now using lemma 1, we get that $P(\sum_{i=1}^n(\sum_{l=1}^n|V_k^T(Y_i-Y_l)|\sum_{l=1}^n|N^X_i-N^X_l|<n^3t_1t_2))\geq 1-n(\alpha_1+\alpha_2-\alpha_1\alpha_2)$. So for any constant $C$, we can write $P(C\sum_{i=1}^n(\sum_{l=1}^n|V_k^T(Y_i-Y_l)|\sum_{l=1}^n|N^X_i-N^X_l|<Cn^3t_1t_2))\geq 1-n(\alpha_1+\alpha_2-\alpha_1\alpha_2)$. Therefore, for each $n$, $ \frac{4C_pC_q}{Kn(n-2)(n-3)}\sum_{i=1}^n\biggl(\sum_{j=1}^n | N_i^X-N_j^X|\sum_{j=1}^n V_K^T(Y_i-Y_j)|\biggr)$ is upper bounded with a high probability. The error term is a summation for all $n$ from $1$ to $K$ and hence the error term is bounded too with a high probability.
\end{proof}

\begin{table*}[]\centering
\begin{tabular}{|c|c|c|}
\hline
\textbf{Dataset}                                                                                  & \textbf{min $tbd$} & \textbf{max$tbd$} \\ \hline
\begin{tabular}[c]{@{}c@{}}Condition Based Maintenance \\ of Naval Propulsion Plants\end{tabular} & 0.0583             & 0.0641            \\ \hline
Parkinsons Telemonitoring                                                                         & 0.0689             & 0.0827            \\ \hline
Wine Quality                                                                                      & 0.03687            & 0.0475            \\ \hline
Boston Housing                                                                                    & 0.0106             & 0.0263            \\ \hline
Madelon                                                                                           & 0.0919             & 0.1046            \\ \hline
Geographic origin of music                                                                        & 0.1784             & 0.2547            \\ \hline
Skill  craft  1  Master  Table                                                                    & 0.1368             & 0.1662            \\ \hline
Seoul bike sharing demand                                                                         & 0.1740             & 0.1844            \\ \hline
\end{tabular}
\end{table*}

\section{Families of some dependency measures}
\label{appdx:Families}
\begin{enumerate}
    \item \textbf{Energy statistics:} Distance correlation (DCOR), Brownian distance covariance, Partial distance correlation, Partial martingale difference correlation
    \item \textbf{Kernel covariance operators:} Constrained covariance (COCO), Hilbert-Schmidt independence criterion (HSIC), Kernel Target Alignment (KTA)
    \item \textbf{Integral probability  metrics:} Maximum mean discrepancey (MMD), Wasserstein distance,  Dudley metric and Fortet-Mourier metric, Total variation distance.
    \item \textbf{Information theoretic measures:} Mutual information, f-divergence, Renyi divergence, Hellinger distance, Total variation distance, Maximal information coefficient, Total information coefficient
\end{enumerate}

\section{Preliminaries for differential privacy}
\begin{definition}
 $(\epsilon, \delta)$-Differential Privacy (2014)
\end{definition}  A randomized algorithm $\mathcal{A}: \mathcal{X} \rightarrow \mathcal{Y}$ is $(\epsilon, \delta)$-differentially private if, for all neighboring datasets $\mathbf{X}, \mathbf{X}^{\prime} \in \mathcal{X}$ and for all $S \in \mathcal{Y}$
$$
\operatorname{Pr}[\mathcal{A}(\mathbf{X}) \in S] \leq e^{\epsilon} \operatorname{Pr}\left[\mathcal{A}\left(\mathbf{X}^{\prime}\right) \in S\right]+\delta
$$
\begin{definition}
 Post-Processing Invariance
\end{definition}
 Differential privacy is immune to post-processing, meaning that an adversary without any additional knowledge about the dataset $\mathbf{X}$ cannot compute a function on the output $\mathcal{A}(\mathbf{X})$ to violate the stated privacy guarantees.

\begin{definition}
 $\left(l_{2}\right.$-Global Sensitivity)
\end{definition} Let $f: \mathcal{X} \rightarrow \mathbb{R}^{k}$. The $l_{2^{-}}$ global sensitivity of $f$ is
$$
\Delta_{2}^{(f)}=\max _{\mathbf{x}, \mathbf{X}^{\prime} \in \mathcal{X}}\left\|f(\mathbf{X})-f\left(\mathbf{X}^{\prime}\right)\right\|_{2}
$$
where $\mathbf{X}, \mathbf{X}^{\prime}$ are neighboring databases.

\section{Differentially private random projections}
\label{app:dprp}
We state one of the classic mechanisms from \cite{Kenthapadi} that is relevant to some aspects of our proposed method. We first share a prerequisite definition that is required prior to re-stating this mechanism. We denote the random projection matrix by $\mathbf{P}$, and note that one of the popular choices for building it is to have each entry of the matrix drawn independently from a Normal distribution with mean $0$ and $\sigma^2 =1/k$. We now define the $\ell_{\rho}$-Sensitivity of $P$.

\begin{definition}
 $\left(\ell_{\rho}\right.$-Sensitivity of $\left.P\right)$. Define the $l_{\rho}$-sensitivity of a $d \times k$ projection matrix $P=\left\{P_{i j}\right\}_{d \times k}$ denoted by $w_{\rho}(P)$, as the maximum $\ell_{\rho}$-norm of any row in $P$, i.e., $w_{\rho}(P)=\max _{1 \leq i \leq d}\left(\sum_{j=1}^{k}\left|P_{i j}\right|^{\rho}\right)^{\frac{1}{\rho}}$
Equivalently, $w_{\rho}(P)$ can be defined as $\max _{e_{i}}\left\|e_{i} P\right\|_{\rho}$, where $\left\{e_{i}\right\}_{i=1}^{d}$ are standard basis unit vectors. 

\end{definition}

 \begin{theorem}
 Let $w_{2}(P)$ be the $\ell_{2}$-sensitivity of the projection matrix $P$ (see Definition 2). Assuming $\delta<\frac{1}{2}$, let the entries of the noise matrix be drawn from $N\left(0, \sigma^{2}\right)$ with $\sigma \geq w_{2}(P) \frac{\sqrt{2\left(\ln \left(\frac{1}{2 \delta}\right)+\epsilon\right)}}{\epsilon}$, then releasing $\mathbf{Z=XP+\Delta}$ satisfies $(\epsilon, \delta)$-differential privacy.
 \end{theorem} 
\begin{proof}
\cite{Kenthapadi}
\end{proof}

\bibliography{references}

\begin{thebibliography}{}

\bibitem[Asghar et~al., 2021]{Asghar_Ding_Rakotoarivelo_Mrabet_Kaafar_2021}
Asghar, H.~J., Ding, M., Rakotoarivelo, T., Mrabet, S., and Kaafar, D. (2021).
\newblock Differentially private release of datasets using gaussian copula.
\newblock {\em Journal of Privacy and Confidentiality}, 10(2).

\bibitem[Blocki et~al., 2012]{blocki2012johnson}
Blocki, J., Blum, A., Datta, A., and Sheffet, O. (2012).
\newblock The johnson-lindenstrauss transform itself preserves differential
  privacy.
\newblock In {\em 2012 IEEE 53rd Annual Symposium on Foundations of Computer
  Science}, pages 410--419. IEEE.

\bibitem[Canonne, 2020]{canonne2020survey}
Canonne, C.~L. (2020).
\newblock A survey on distribution testing: Your data is big. but is it blue?
\newblock {\em Theory of Computing}, pages 1--100.

\bibitem[Chaudhuri and Hu, 2019]{chaudhuri2019fast}
Chaudhuri, A. and Hu, W. (2019).
\newblock A fast algorithm for computing distance correlation.
\newblock {\em Computational Statistics \& Data Analysis}.

\bibitem[Dwork, 2006]{dwork2006differential}
Dwork, C. (2006).
\newblock Differential privacy.
\newblock In {\em International Colloquium on Automata, Languages, and
  Programming}, pages 1--12. Springer.

\bibitem[Dwork, 2008]{dwork2008differential}
Dwork, C. (2008).
\newblock Differential privacy: A survey of results.
\newblock In {\em International conference on theory and applications of models
  of computation}, pages 1--19. Springer.

\bibitem[Dwork et~al., 2014]{dwork2014algorithmic}
Dwork, C., Roth, A., et~al. (2014).
\newblock The algorithmic foundations of differential privacy.
\newblock {\em Found. Trends Theor. Comput. Sci.}, 9(3-4):211--407.

\bibitem[Fan and Lv, 2018]{fan2018sure}
Fan, J. and Lv, J. (2018).
\newblock Sure independence screening.
\newblock {\em Wiley StatsRef: Statistics Reference Online}.

\bibitem[Gaboardi et~al., 2016]{gaboardi2016differentially}
Gaboardi, M., Lim, H., Rogers, R., and Vadhan, S. (2016).
\newblock Differentially private chi-squared hypothesis testing: Goodness of
  fit and independence testing.
\newblock In {\em International conference on machine learning}, pages
  2111--2120. PMLR.

\bibitem[Gondara and Wang, 2020]{gondara2020differentially}
Gondara, L. and Wang, K. (2020).
\newblock Differentially private small dataset release using random
  projections.
\newblock In {\em Conference on Uncertainty in Artificial Intelligence}, pages
  639--648. PMLR.

\bibitem[Gretton et~al., 2005]{gretton2005measuring}
Gretton, A., Bousquet, O., Smola, A., and Sch{\"o}lkopf, B. (2005).
\newblock Measuring statistical dependence with hilbert-schmidt norms.
\newblock In {\em International conference on algorithmic learning theory},
  pages 63--77. Springer.

\bibitem[Huang and Huo, 2017]{huang2017statistically}
Huang, C. and Huo, X. (2017).
\newblock A statistically and numerically efficient independence test based on
  random projections and distance covariance.
\newblock {\em arXiv preprint arXiv:1701.06054}.

\bibitem[Huo and Sz{\'e}kely, 2016]{huo2016fast}
Huo, X. and Sz{\'e}kely, G.~J. (2016).
\newblock Fast computing for distance covariance.
\newblock {\em Technometrics}, 58(4):435--447.

\bibitem[Kenthapadi et~al., 2013]{Kenthapadi}
Kenthapadi, K., Korolova, A., Mironov, I., and Mishra, N. (2013).
\newblock Privacy via the johnson-lindenstrauss transform.
\newblock {\em Journal of Privacy and Confidentiality}, 5(1).

\bibitem[Kusner et~al., 2015]{kusner2015inferring}
Kusner, M.~J., Sun, Y., Sridharan, K., and Weinberger, K.~Q. (2015).
\newblock Inferring the causal direction privately.
\newblock {\em stat}, 1050:17.

\bibitem[Kusner et~al., 2016]{kusner2016private}
Kusner, M.~J., Sun, Y., Sridharan, K., and Weinberger, K.~Q. (2016).
\newblock Private causal inference.
\newblock In {\em Artificial Intelligence and Statistics}, pages 1308--1317.
  PMLR.

\bibitem[Li et~al., 2014]{li2014differentially}
Li, H., Xiong, L., and Jiang, X. (2014).
\newblock Differentially private synthesization of multi-dimensional data using
  copula functions.
\newblock In {\em Advances in database technology: proceedings. International
  conference on extending database technology}, volume 2014, page 475. NIH
  Public Access.

\bibitem[Li et~al., 2012]{li2012feature}
Li, R., Zhong, W., and Zhu, L. (2012).
\newblock Feature screening via distance correlation learning.
\newblock {\em Journal of the American Statistical Association},
  107(499):1129--1139.

\bibitem[Lu et~al., 2021]{lu2021conditional}
Lu, S., Chen, X., and Wang, H. (2021).
\newblock Conditional distance correlation sure independence screening for
  ultra-high dimensional survival data.
\newblock {\em Communications in Statistics-Theory and Methods},
  50(8):1936--1953.

\bibitem[Rizzo and Sz{\'e}kely, 2016]{rizzo2016energy}
Rizzo, M.~L. and Sz{\'e}kely, G.~J. (2016).
\newblock Energy distance.
\newblock {\em wiley interdisciplinary reviews: Computational statistics},
  8(1):27--38.

\bibitem[Shen et~al., 2019]{shen2019exact}
Shen, C., Priebe, C.~E., and Vogelstein, J.~T. (2019).
\newblock The exact equivalence of independence testing and two-sample testing.
\newblock {\em arXiv preprint arXiv:1910.08883}.

\bibitem[Shen and Vogelstein, 2021]{shen2021exact}
Shen, C. and Vogelstein, J.~T. (2021).
\newblock The exact equivalence of distance and kernel methods in hypothesis
  testing.
\newblock {\em AStA Advances in Statistical Analysis}, 105(3):385--403.

\bibitem[Swanberg et~al., 2019]{swanberg}
Swanberg, M., Globus-Harris, I., Griffith, I., Ritz, A., Groce, A., and Bray,
  A. (2019).
\newblock Improved differentially private analysis of variance.
\newblock {\em Proceedings on Privacy Enhancing Technologies},
  2019(3):310--330.

\bibitem[Sz{\'e}kely and Rizzo, 2013a]{szekely2013distance}
Sz{\'e}kely, G.~J. and Rizzo, M.~L. (2013a).
\newblock The distance correlation t-test of independence in high dimension.
\newblock {\em Journal of Multivariate Analysis}, 117:193--213.

\bibitem[Sz{\'e}kely and Rizzo, 2013b]{szekely2013energy}
Sz{\'e}kely, G.~J. and Rizzo, M.~L. (2013b).
\newblock Energy statistics: A class of statistics based on distances.
\newblock {\em Journal of statistical planning and inference},
  143(8):1249--1272.

\bibitem[Sz{\'e}kely and Rizzo, 2017]{szekely2017energy}
Sz{\'e}kely, G.~J. and Rizzo, M.~L. (2017).
\newblock The energy of data.
\newblock {\em Annual Review of Statistics and Its Application}, 4:447--479.

\bibitem[Sz{\'e}kely et~al., 2007]{szekely2007measuring}
Sz{\'e}kely, G.~J., Rizzo, M.~L., and Bakirov, N.~K. (2007).
\newblock Measuring and testing dependence by correlation of distances.
\newblock {\em The annals of statistics}, 35(6):2769--2794.

\bibitem[Thakurta and Smith, 2013]{thakurta2013differentially}
Thakurta, A.~G. and Smith, A. (2013).
\newblock Differentially private feature selection via stability arguments, and
  the robustness of the lasso.
\newblock In {\em Conference on Learning Theory}, pages 819--850. PMLR.

\bibitem[Upadhyay, 2013]{upadhyay2013random}
Upadhyay, J. (2013).
\newblock Random projections, graph sparsification, and differential privacy.
\newblock In {\em International Conference on the Theory and Application of
  Cryptology and Information Security}, pages 276--295. Springer.

\bibitem[Xu et~al., 2017]{xu2017dppro}
Xu, C., Ren, J., Zhang, Y., Qin, Z., and Ren, K. (2017).
\newblock Dppro: Differentially private high-dimensional data release via
  random projection.
\newblock {\em IEEE Transactions on Information Forensics and Security},
  12(12):3081--3093.

\bibitem[Zhong and Zhu, 2015]{zhong2015iterative}
Zhong, W. and Zhu, L. (2015).
\newblock An iterative approach to distance correlation-based sure independence
  screening.
\newblock {\em Journal of Statistical Computation and Simulation},
  85(11):2331--2345.

\end{thebibliography}

\end{document}